\documentclass[11pt]{article}

\usepackage{epsfig,epsf,fancybox}
\usepackage{amsmath}
\usepackage{mathrsfs}
\usepackage{amssymb}
\usepackage{graphicx}
\usepackage{color}
\usepackage{multirow}
\usepackage{paralist}
\usepackage{verbatim}
\usepackage{galois}
\usepackage{algorithm}
\usepackage{algorithmic}
\usepackage{boxedminipage}
\usepackage{booktabs}
\usepackage{accents}
\usepackage{stmaryrd}
\usepackage{subfig}
\usepackage{wrapfig}
\usepackage[bottom]{footmisc}
\usepackage{natbib}

\usepackage[colorlinks,linkcolor=magenta,citecolor=blue, pagebackref=true,backref=true]{hyperref}
\renewcommand*{\backrefalt}[4]{%
    \ifcase #1 \footnotesize{(Not cited.)}%
    \or        \footnotesize{(Cited on page~#2.)}%
    \else      \footnotesize{(Cited on pages~#2.)}%
    \fi}

\long\def\comment#1{}
\usepackage{nicefrac}

\newtheorem{theorem}{Theorem}[section]

\newtheorem{lemma}[theorem]{Lemma}
\newtheorem{proposition}[theorem]{Proposition}

\newtheorem{definition}{Definition}[section]

\newtheorem{remark}[theorem]{Remark}

\numberwithin{equation}{section}

\newcommand{\st}{\textnormal{s.t.}}

\newcommand{\diag}{\textnormal{diag}}
\newcommand{\grad}{\textnormal{grad}\,}

\newcommand{\subdiff}{\textnormal{subdiff}\,}
\newcommand{\Diag}{\textnormal{Diag}\,}

\newcommand{\retr}{\textnormal{Retr}}

\newcommand{\sign}{\textnormal{sign}}

\newcommand{\argmin}{\mathop{\rm argmin}}
\newcommand{\argmax}{\mathop{\rm argmax}}

\newcommand{\GCal}{\mathcal{G}}

\newcommand{\SCal}{\mathcal{S}}
\newcommand{\XCal}{\mathcal{X}}
\newcommand{\YCal}{\mathcal{Y}}

\newcommand{\PScr}{\mathscr{P}}

\newcommand{\trac}{\textnormal{Trace}}
\newcommand{\prox}{\textnormal{prox}}
\newcommand{\proj}{\textnormal{proj}}

\newcommand{\dist}{\textnormal{dist}}
\newcommand{\St}{\textnormal{St}}
\newcommand{\Tg}{\textnormal{T}}

\newcommand{\br}{\mathbb{R}}

\newcommand{\ba}{\begin{array}}
\newcommand{\ea}{\end{array}}

\newcommand{\WCal}{\mathcal{W}}

\newcommand{\RCal}{\mathcal{R}}
\newcommand{\PCal}{\mathcal{P}}

\newcommand{\NCal}{\mathcal{N}}

\newcommand{\UCal}{\mathcal{U}}
\newcommand{\one}{\textbf{1}}
\newcommand{\zero}{\textbf{0}}
\newcommand{\bigO}{O}
\newcommand{\bigOtil}{\widetilde{O}}
\newcommand{\mydefn}{:=}

\newcommand{\sk}{Shakespeare\ }

\begin{document}

\begin{center}

{\bf{\LARGE{Projection Robust Wasserstein Distance and \\ [.2cm] Riemannian Optimization}}}

\vspace*{.2in}
{\large{
\begin{tabular}{ccccc}
Tianyi Lin$^{\star, \diamond}$ & Chenyou Fan$^{\star, \circ, \square}$ & Nhat Ho$^\ddagger$ & Marco Cuturi$^{\triangleleft, \triangleright}$ & Michael I. Jordan$^{\diamond, \dagger}$
\end{tabular}
}}

\vspace*{.2in}

\begin{tabular}{c}
Department of Electrical Engineering and Computer Sciences$^\diamond$ \\
Department of Statistics$^\dagger$ \\ 
University of California, Berkeley \\
The Chinese University of Hong Kong, Shenzhen$^\square$ \\
Department of Statistics and Data Sciences, University of Texas, Austin$^\ddagger$ \\ 
CREST - ENSAE$^\triangleleft$, Google Brain$^\triangleright$
\end{tabular}

\vspace*{.2in}

\today

\vspace*{.2in}

\begin{abstract}
Projection robust Wasserstein (PRW) distance, or Wasserstein projection pursuit (WPP), is a robust variant of the Wasserstein distance. Recent work suggests that this quantity is more robust than the standard Wasserstein distance, in particular when comparing probability measures in high-dimensions. However, it is ruled out for practical application because the optimization model is essentially non-convex and non-smooth which makes the computation intractable. Our contribution in this paper is to revisit the original motivation behind WPP/PRW, but take the hard route of showing that, despite its non-convexity and lack of nonsmoothness, and even despite some hardness results proved by~\citet{Niles-2019-Estimation} in a minimax sense, the original formulation for PRW/WPP \textit{can} be efficiently computed in practice using Riemannian optimization, yielding in relevant cases better behavior than its convex relaxation. More specifically, we provide three simple algorithms with solid theoretical guarantee on their complexity bound (one in the appendix), and demonstrate their effectiveness and efficiency by conducting extensive experiments on synthetic and real data. This paper provides a first step into a computational theory of the PRW distance and provides the links between optimal transport and Riemannian optimization. 
\end{abstract}
\let\thefootnote\relax\footnotetext{$^*$ Tianyi Lin and Chenyou Fan contributed equally to this work.}
\let\thefootnote\relax\footnotetext{$^\circ$ Chenyou Fan contributed during working at Google.}
\end{center}

\section{Introduction}
Optimal transport (OT) theory~\citep{Villani-2003-Topics, Villani-2008-Optimal} has become an important source of ideas and algorithmic tools in machine learning and related fields. Examples include contributions to generative modelling~\citep{Arjovsky-2017-Wasserstein, Salimans-2018-Improving, Tolstikhin-2018-Wasserstein, Genevay-2018-Learning}, domain adaptation~\citep{Courty-2017-Optimal}, clustering~\citep{Srivastava-2015-WASP, Ho-2017-Multilevel}, dictionary learning~\citep{Rolet-2016-Fast, Schmitz-2018-Wasserstein}, text mining~\citep{Lin-2019-Sparsemax}, neuroimaging~\citep{Janati-2020-Multi} and single-cell genomics~\citep{Schiebinger-2019-Optimal, Yang-2020-Predicting}. The Wasserstein geometry has also provided a simple and useful analytical tool to study latent mixture models~\citep{Ho-2016-Convergence}, reinforcement learning~\citep{Bellemare-2017-Distributional}, sampling~\citep{Cheng-2018-Underdamped, Dalalyan-2019-User, Mou-2019-High, Bernton-2018-Langevin} and stochastic optimization~\citep{Nagaraj-2019-Sgd}. For an overview of OT theory and the relevant applications, we refer to the recent survey~\citep{Peyre-2019-Computational}. 

\paragraph{Curse of Dimensionality in OT.} A significant barrier to the direct application of OT in machine learning lies in some inherent statistical limitations. It is well known that the sample complexity of approximating Wasserstein distances between densities using only samples can grow exponentially in dimension~\citep{Dudley-1969-Speed, Fournier-2015-Rate, Weed-2019-Sharp, Lei-2020-Convergence}. Practitioners have long been aware of this issue of the curse of dimensionality in applications of OT, and it can be argued that most of the efficient computational schemes that are known to improve computational complexity also carry out, implicitly through their simplifications, some form of statistical regularization. There have been many attempts to mitigate this curse when using OT, whether through entropic regularization~\citep{Cuturi-2013-Sinkhorn, Cuturi-2014-Barycenter, Genevay-2019-Sample, Mena-2019-Statistical}; other regularizations~\citep{Dessein-2018-Regularized, Blondel-2018-Smooth}; quantization~\citep{Canas-2012-Learning, Forrow-2019-Statistical}; simplification of the dual problem in the case of 1-Wasserstein distance~\citep{Shirdhonkar-2008-Approximate, Arjovsky-2017-Wasserstein} or by only using second-order moments of measures to fall back on the Bures-Wasserstein distance~\citep{Bhatia-2018-Bures, Muzellec-2018-Elliptical, Chen-2018-Optimal}. 

\paragraph{Subspace projections: PRW and WPP.} We focus in this paper on another important approach to regularize the Wasserstein distance: Project input measures onto lower-dimensional subspaces and compute the Wasserstein distance between these reductions, instead of the original measures. The simplest and most representative example of this approach is the sliced Wasserstein distance~\citep{Rabin-2011-Wasserstein, Bonneel-2015-Sliced, Kolouri-2019-Generalized, Nguyen-2020-Distributional}, which is defined as the average Wasserstein distance obtained between random 1D projections. In an important extension,~\citet{Paty-2019-Subspace} and~\citet{Niles-2019-Estimation} proposed very recently to look for the $k$-dimensional subspace ($k>1$) that would maximize the Wasserstein distance between two measures after projection.~\citep{Paty-2019-Subspace} called that quantity the \textit{projection robust Wasserstein} (PRW) distance, while~\citet{Niles-2019-Estimation} named it \textit{Wasserstein Projection Pursuit} (WPP). PRW/WPP are conceptually simple, easy to interpret, and do solve the curse of dimensionality in the so called spiked model as proved in~\citet[Theorem 1]{Niles-2019-Estimation} by recovering an optimal $1/\sqrt{n}$ rate. Very recently,~\citet{Lin-2021-Projection} further provided several fundamental statistical bounds for PRW as well as asymptotic guarantees for learning generative models with PRW. Despite this appeal,~\citep{Paty-2019-Subspace} quickly rule out PRW for practical applications because it is non-convex, and fall back on a convex relaxation, called the \emph{subspace robust Wasserstein} (SRW) distance, which is shown to work better empirically than the usual Wasserstein distance. Similarly,~\citet{Niles-2019-Estimation} seem to lose hope that it can be computed, by stating \textit{``it is unclear how to implement WPP efficiently,''} and after having proved positive results on sample complexity, conclude their paper on a negative note, showing hardness results which apply for WPP when the ground cost is the Euclidean metric (the 1-Wasserstein case). Our contribution in this paper is to revisit the original motivation behind WPP/PRW, but take the hard route of showing that, despite its non-convexity and lack of nonsmoothness, and even despite some hardness results proved in~\cite{Niles-2019-Estimation} in a minimax sense, the original formulation for PRW/WPP \textit{can} be efficiently computed in practice using Riemannian optimization, yielding in relevant cases better behavior than SRW. For simplicity, we refer from now on to PRW/WPP as PRW.

\paragraph{Contribution:} In this paper, we study the computation of the PRW distance between two discrete probability measures of size $n$. We show that the resulting optimization problem has a special structure, allowing it to be solved in an efficient manner using Riemannian optimization~\citep{Absil-2009-Optimization, Boumal-2019-Global, Kasai-2019-Riemannian, Chen-2020-Proximal}. Our contributions can be summarized as follows. 
\begin{enumerate}
\item We propose a max-min optimization model for computing the PRW distance. The maximization and minimization are performed over the Stiefel manifold and the transportation polytope, respectively. We prove the existence of the subdifferential (Lemma~\ref{lemma:bound-subdiff}), which allows us to properly define an \emph{$\epsilon$-approximate pair of optimal subspace projection and optimal transportation plan} (Definition~\ref{def:optimal-pair}) and carry out a finite-time analysis of the algorithm. 

\item We define an entropic regularized PRW distance between two finite discrete probability measures, and show that it is possible to efficiently optimize this distance over the transportation polytope using the Sinkhorn iteration. This poses the problem of performing the maximization over the Stiefel manifold, which is not solvable by existing optimal transport algorithms~\citep{Cuturi-2013-Sinkhorn, Altschuler-2017-Near, Dvurechensky-2018-Computational, Lin-2019-Efficient, Lin-2019-Acceleration, Guminov-2019-Accelerated}. To this end, we propose two new algorithms, which we refer to as \textit{Riemannian gradient ascent with Sinkhorn} (RGAS) and \textit{Riemannian adaptive gradient ascent with Sinkhorn} (RAGAS), for computing the entropic regularized PRW distance. These two algorithms are guaranteed to return an \emph{$\epsilon$-approximate pair of optimal subspace projection and optimal transportation plan} with a complexity bound of $\bigOtil(n^2d\|C\|_\infty^4\epsilon^{-4} + n^2\|C\|_\infty^8\epsilon^{-8} + n^2\|C\|_\infty^{12}\epsilon^{-12})$. To the best of our knowledge, our algorithms are the first provably efficient algorithms for the computation of the PRW distance. 

\item We provide comprehensive empirical studies to evaluate our algorithms on synthetic and real datasets. Experimental results confirm our conjecture that the PRW distance performs better than its convex relaxation counterpart, the SRW distance. Moreover, we show that the RGAS and RAGAS algorithms are faster than the Frank-Wolfe algorithm while the RAGAS algorithm is more robust than the RGAS algorithm. 
\end{enumerate}

\paragraph{Organization.} The remainder of the paper is organized as follows. In Section~\ref{sec:background}, we present the nonconvex max-min optimization model for computing the PRW distance and its entropic regularized version. We also briefly summarize various concepts of geometry and optimization over the Stiefel manifold. In Section~\ref{sec:algorithm}, we propose and analyze the RGAS and RAGAS algorithms for computing the entropic regularized PRW distance and prove that both algorithms achieve the finite-time guarantee under stationarity measure. In Section~\ref{sec:experiment}, we conduct extensive experiments on both synthetic and real datasets, demonstrating that the PRW distance provides a computational advantage over the SRW distance in real application problems. In the supplementary material, we provide further background materials on Riemannian optimization, experiments with the algorithms, and proofs for key results. For the sake of completeness, we derive a near-optimality condition (Definition~\ref{def:near-stationarity} and~\ref{def:near-optimal-pair}) for the max-min optimization model and propose another \emph{Riemannian SuperGradient Ascent with Network simplex iteration} (RSGAN) algorithm for computing the PRW distance without regularization and prove the finite-time convergence under the near-optimality condition.  

\paragraph{Notation.} We let $[n]$ be the set $\{1, 2, \ldots, n\}$ and $\br^n_+$ be the set of all vectors in $\br^n$ with nonnegative components. $\one_n$ and $\zero_n$ are the $n$-dimensional vectors of ones and zeros. $\Delta^n = \{u \in \br^n_+: \one_n^\top u = 1\}$ is the probability simplex. For a vector $x \in \br^n$, the Euclidean norm stands for $\|x\|$ and the Dirac delta function at $x$ stands for $\delta_x(\cdot)$. The notation $\Diag(x)$ denotes an $n \times n$ diagonal matrix with $x$ as the diagonal elements. For a matrix $X \in \br^{n \times n}$, the right and left marginals are denoted $r(X) = X\one_n$ and $c(X) = X^\top\one_n$, and $\|X\|_\infty = \max_{1 \leq i, j \leq n} |X_{ij}|$ and $\|X\|_1 = \sum_{1 \leq i, j \leq n} |X_{ij}|$. The notation $\diag(X)$ stands for an $n$-dimensional vector which corresponds to the diagonal elements of $X$. If $X$ is symmetric, $\lambda_{\max}(X)$ stands for largest eigenvalue. The notation $\St(d, k) \mydefn \{X \in \br^{d \times k}: X^\top X = I_k\}$ denotes the Stiefel manifold. For $X, Y \in \br^{n \times n}$, we denote $\langle X, Y\rangle = \trac(X^\top Y)$ as the Euclidean inner product and $\|X\|_F$ as the Frobenius norm of $X$. We let $P_\SCal$ be the orthogonal projection onto a closed set $\SCal$ and $\dist(X, \SCal) = \inf_{Y \in \SCal} \|X-Y\|_F$ denotes the distance between $X$ and $\SCal$. Lastly, $a = \bigO(b(n, d, \epsilon))$ stands for the upper bound $a \leq C \cdot b(n, d, \epsilon)$ where $C>0$ is independent of $n$ and $1/\epsilon$ and $a = \bigOtil(b(n, d, \epsilon))$ indicates the same inequality where $C$ depends on the logarithmic factors of $n$, $d$ and $1/\epsilon$.  

\section{Projection Robust Wasserstein Distance}\label{sec:background}
In this section, we present the basic setup and optimality conditions for the computation of the projection robust 2-Wasserstein (PRW) distance between two discrete probability measures with at most $n$ components. We also review basic ideas in Riemannian optimization. 

\subsection{Structured max-min optimization model}
In this section we define the PRW distance~\citep{Paty-2019-Subspace} and show that computing the PRW distance between two discrete probability measures supported on at most $n$ points reduces to solving a structured max-min optimization model over the Stiefel manifold and the transportation polytope. 

Let $\PScr(\br^d)$ be the set of Borel probability measures in $\br^d$ and let $\PScr_2(\br^d)$ be the subset of $\PScr(\br^d)$ consisting of probability measures that have finite second moments. Let $\mu, \nu \in \PScr_2(\br^d)$ and $\Pi(\mu, \nu)$ be the set of couplings between $\mu$ and $\nu$. The 2-Wasserstein distance~\citep{Villani-2008-Optimal} is defined by 
\begin{equation*}
\WCal_2(\mu, \nu) \ \mydefn \ \left(\inf_{\pi \in \Pi(\mu, \nu)} \int \|x-y\|^2 \ d\pi(x, y)\right)^{1/2}. 
\end{equation*}
To define the PRW distance, we require the notion of the push-forward of a measure by an operator. Letting $\XCal, \YCal \subseteq \br^d$ and $T: \XCal \rightarrow \YCal$, the push-forward of $\mu \in \PScr(\XCal)$ by $T$ is defined by $T_{\#} \mu \in \mathscr{P}(\YCal)$. In other words, $T_{\#} \mu$ is the measure satisfying $T_{\#} \mu(A) = \mu(T^{-1}(A))$ for any Borel set in $\YCal$.
\begin{definition}
For $\mu, \nu \in \PScr_2(\br^d)$, let $\GCal_k = \{E \subseteq \br^d \mid \dim(E) = k\}$ be the Grassmannian of $k$-dimensional subspace of $\br^d$ and let $P_E$ be the orthogonal projector onto $E$ for all $E \in \GCal_k$.  The $k$-dimensional PRW distance is defined as $\PCal_k(\mu, \nu) \mydefn \sup_{E \in \GCal_k} \WCal_2(P_{E\#}\mu, P_{E\#}\nu)$. 
\end{definition}
\citet[Proposition~5]{Paty-2019-Subspace} have shown that there exists a subspace $E^* \in \GCal_k$ such that $\PCal_k(\mu, \nu) = \WCal_2(P_{E^*\#}\mu, P_{E^*\#}\nu)$ for any $k \in [d]$ and $\mu, \nu \in \PScr_2(\br^d)$. For any $E \in \GCal_k$, the mapping $\pi \mapsto \int \|P_E(x-y)\|^2 \ d\pi(x, y)$ is lower semi-continuous. This together with the compactness of $\Pi(\mu, \nu)$ implies that the infimum is a minimum. Therefore, we obtain a structured max-min optimization problem:
\begin{equation*}
\PCal_k(\mu, \nu) = \max_{E \in \GCal_k} \min_{\pi \in \Pi(\mu, \nu)} \left(\int \|P_E(x-y)\|^2 \ d\pi(x, y)\right)^{1/2}. 
\end{equation*}
Let us now consider this general problem in the case of discrete probability measures, which is the focus of the current paper. Let $\{x_1, x_2, \ldots, x_n\} \subseteq \br^d$ and $\{y_1, y_2, \ldots, y_n\} \subseteq \br^d$ denote sets of $n$ atoms, and let $(r_1, r_2, \ldots, r_n) \in \Delta^n$ and $(c_1, c_2, \ldots, c_n) \in \Delta^n$ denote weight vectors.  We define discrete probability measures $\mu \mydefn \sum_{i=1}^n r_i\delta_{x_i}$ and $\nu \mydefn \sum_{j=1}^n c_j\delta_{y_j}$. In this setting, the computation of the $k$-dimensional PRW distance between $\mu$ and $\nu$ reduces to solving a structured max-min optimization model where the maximization and minimization are performed over the Stiefel manifold $\St(d, k) : = \{U \in \br^{d \times k} \mid U^\top U = I_k\}$ and the transportation polytope $\Pi(\mu, \nu) : = \{\pi \in \br_+^{n \times n} \mid r(\pi) = r, \ c(\pi) = c\}$ respectively. Formally, we have
\begin{equation}\label{prob:main}
\max\limits_{U \in \br^{d \times k}} \min\limits_{\pi \in \br_+^{n \times n}} \sum_{i=1}^n \sum_{j=1}^n \pi_{i, j}\|U^\top x_i - U^\top y_j\|^2 \quad \st \ U^\top U = I_k, \ r(\pi) = r, \ c(\pi) = c.
\end{equation}
The computation of this PRW distance raises numerous challenges. Indeed, there is no guarantee for finding a global Nash equilibrium as the special case of nonconvex optimization is already NP-hard~\citep{Murty-1987-Some}; moreover, Sion's minimax theorem~\citep{Sion-1958-General} is not applicable here due to the lack of quasi-convex-concave structure. More practically, solving Eq.~\eqref{prob:main} is expensive since (i) preserving the orthogonality constraint requires the singular value decompositions (SVDs) of a $d \times d$ matrix, and (ii) projecting onto the transportation polytope results in a costly quadratic network flow problem. To avoid this, ~\citep{Paty-2019-Subspace} proposed a convex surrogate for Eq.~\eqref{prob:main}: 
\begin{equation}\label{prob:main-CCP}
\max\limits_{0 \preceq \Omega \preceq I_d} \min\limits_{\pi \in \br_+^{n \times n}} \sum_{i=1}^n \sum_{j=1}^n \pi_{i, j}(x_i - y_j)^\top\Omega(x_i - y_j), \quad \st \ \trac(\Omega)=k, \ r(\pi) = r, \ c(\pi) = c.
\end{equation}
Eq.~\eqref{prob:main-CCP} is intrinsically a bilinear minimax optimization model which makes the computation tractable. Indeed, the constraint set $\RCal = \{\Omega \in \br^{d \times d} \mid 0 \preceq \Omega \preceq I_d, \trac(\Omega)=k\}$ is convex and the objective function is bilinear since it can be rewritten as $\langle \Omega, \sum_{i=1}^n \sum_{j=1}^n \pi_{i, j} (x_i - y_j)(x_i - y_j)^\top\rangle$.  Eq.~\eqref{prob:main-CCP} is, however, only a convex relaxation of Eq.~\eqref{prob:main} and its solutions are not necessarily good approximate solutions for the original problem. Moreover, the existing algorithms for solving Eq.~\eqref{prob:main-CCP} are also unsatisfactory---in each loop, we need to solve a OT or entropic regularized OT exactly and project a $d \times d$ matrix onto the set $\RCal$ using the SVD decomposition, both of which are computationally expensive as $d$ increases (see Algorithm~1 and~2 in \cite{Paty-2019-Subspace}).
\begin{algorithm}[!t]
\caption{Riemannian Gradient Ascent with Sinkhorn Iteration (RGAS)}\label{alg:grad-sinkhorn}
\begin{algorithmic}[1]
\STATE \textbf{Input:} $\{(x_i, r_i)\}_{i \in [n]}$ and $\{(y_j, c_j)\}_{j \in [n]}$, $k = \bigOtil(1)$, $U_0 \in \St(d, k)$ and $\epsilon$.  
\STATE \textbf{Initialize:} $\widehat{\epsilon} \leftarrow \frac{\epsilon}{10\|C\|_\infty}$, $\eta \leftarrow \frac{\epsilon\min\{1, 1/\bar{\theta}\}}{40\log(n)}$ and $\gamma \leftarrow \frac{1}{(8 L_1^2 + 16L_2)\|C\|_\infty + 16 \eta^{-1}L_1^2\|C\|_\infty^2}$. 
\FOR{$t = 0, 1, 2, \ldots$}
\STATE Compute $\pi_{t+1} \leftarrow \textsc{regOT}(\{(x_i, r_i)\}_{i \in [n]}, \{(y_j, c_j)\}_{j \in [n]}, U_t, \eta, \widehat{\epsilon})$.
\STATE Compute $\xi_{t+1} \leftarrow P_{\Tg_{U_t}\St}(2V_{\pi_{t+1}}U_t)$. 
\STATE Compute $U_{t+1} \leftarrow \retr_{U_t}(\gamma\xi_{t+1})$.  
\ENDFOR
\end{algorithmic}
\end{algorithm}

\subsection{Entropic regularized projection robust Wasserstein}
Eq.~\eqref{prob:main} has structure that can be exploited. Indeed, fixing a $U \in \St(d, k)$, the problem reduces to minimizing a linear function over the transportation polytope, i.e., the OT problem. Therefore, we can reformulate Eq.~\eqref{prob:main} as the maximization of the function $f(U) \mydefn \min_{\pi \in \Pi(\mu, \nu)} \sum_{i=1}^n \sum_{j=1}^n \pi_{i, j} \|U^\top x_i - U^\top y_j\|^2$ over the Stiefel manifold $\St(d, k)$. 

Since the OT problem admits multiple optimal solutions, $f$ is not differentiable which makes the optimization over the Stiefel manifold hard~\citep{Absil-2019-Collection}. Computations are greatly facilitated by adding smoothness, which allows the use of gradient-type and adaptive gradient-type algorithms. This inspires us to consider an entropic regularized version of Eq.~\eqref{prob:main}, where an entropy penalty is added to the PRW distance. The resulting optimization model is as follows: 
\begin{equation}\label{prob:main-regularized}
\max\limits_{U \in \br^{d \times k}} \min\limits_{\pi \in \br_+^{n \times n}} \sum_{i=1}^n \sum_{j=1}^n \pi_{i, j}\|U^\top x_i - U^\top y_j\|^2 - \eta H(\pi) \quad \st \ U^\top U = I_k, \ r(\pi) = r, \ c(\pi) = c,
\end{equation}
where $\eta > 0$ is the regularization parameter and $H(\pi) \mydefn - \langle \pi, \log(\pi) - \one_n\one_n^\top\rangle$ denotes the entropic regularization term. We refer to Eq.~\eqref{prob:main-regularized} as the computation of \emph{entropic regularized PRW distance}. Accordingly, we define the function $f_\eta = \min_{\pi \in \Pi(\mu, \nu)} \{\sum_{i=1}^n \sum_{j=1}^n \pi_{i, j} \|U^\top x_i - U^\top y_j\|^2 - \eta H(\pi)\}$ and reformulate Eq.~\eqref{prob:main-regularized} as the maximization of the differentiable function $f_\eta$ over the Stiefel manifold $\St(d, k)$. Indeed, for any $U \in \St(d, k)$ and a fixed $\eta > 0$, there exists a unique solution $\pi^* \in \Pi(\mu, \nu)$ such that $\pi \mapsto \sum_{i=1}^n \sum_{j=1}^n \pi_{i, j}\|U^\top x_i - U^\top y_j\|^2 - \eta H(\pi)$ is minimized at $\pi^*$. When $\eta$ is large, the optimal value of Eq.~\eqref{prob:main-regularized} may yield a poor approximation of Eq.~\eqref{prob:main}. To guarantee a good approximation, we scale the regularization parameter $\eta$ as a function of the desired accuracy of the approximation. Formally, we consider the following relaxed optimality condition for $\widehat{\pi} \in \Pi(\mu, \nu)$ given $U \in \St(d, k)$. 
\begin{definition}\label{def:approx_transportation_plan}
The transportation plan $\widehat{\pi} \in \Pi(\mu, \nu)$ is called an \emph{$\epsilon$-approximate optimal transportation plan for a given $U \in \St(d, k)$} if the following inequality holds: 
\begin{equation*}
\sum_{i=1}^n \sum_{j=1}^n \widehat{\pi}_{i, j} \|U^\top x_i - U^\top y_j\|^2 \ \leq \ \min\limits_{\pi \in \Pi(\mu, \nu)} \sum_{i=1}^n \sum_{j=1}^n \pi_{i, j} \|U^\top x_i - U^\top y_j\|^2 + \epsilon. 
\end{equation*}
\end{definition}

\subsection{Optimality condition}
Recall that the computation of the PRW distance in Eq.~\eqref{prob:main} and the entropic regularized PRW distance in Eq.~\eqref{prob:main-regularized} are equivalent to
\begin{equation}\label{prob:Stiefel-nonsmooth}
\max\limits_{U \in \St(d, k)} \ \left\{f(U) \mydefn \min\limits_{\pi \in \Pi(\mu, \nu)} \sum_{i=1}^n \sum_{j=1}^n \pi_{i, j} \|U^\top x_i - U^\top y_j\|^2\right\},  
\end{equation}
and
\begin{equation}\label{prob:Stiefel-smooth}
\max\limits_{U \in \St(d, k)} \ \left\{f_\eta(U) \mydefn \min\limits_{\pi \in \Pi(\mu, \nu)} \sum_{i=1}^n \sum_{j=1}^n \pi_{i, j} \|U^\top x_i - U^\top y_j\|^2 - \eta H(\pi)\right\}.   
\end{equation}
Since $\St(d, k)$ is a compact matrix submanifold of $\br^{d \times k}$~\citep{Boothby-1986-Introduction}, Eq.~\eqref{prob:Stiefel-nonsmooth} and Eq.~\eqref{prob:Stiefel-smooth} are both special instances of the Stiefel manifold optimization problem. The dimension of $\St(d, k)$ is equal to $dk - k(k+1)/2$ and the tangent space at the point $Z \in \St(d, k)$ is defined by $\Tg_Z\St \mydefn \{\xi \in \br^{d \times k}: \xi^\top Z + Z^\top\xi = 0\}$. We endow $\St(d, k)$ with Riemannian metric inherited from the Euclidean inner product $\langle X, Y\rangle$ for any $X, Y \in \Tg_Z\St$ and $Z \in \St(d, k)$. Then the projection of $G \in \br^{d \times k}$ onto $\Tg_Z\St$ is given by~\citet[Example~3.6.2]{Absil-2009-Optimization}: $P_{\Tg_Z\St}(G) = G - Z(G^\top Z + Z^\top G)/2$. We make use of the notion of a \textit{retraction}, which is the first-order approximation of an exponential mapping on the manifold and which is amenable to computation~\citep[Definition~4.1.1]{Absil-2009-Optimization}. For the Stiefel manifold, we have the following definition: 
\begin{definition}\label{def:retraction}
A retraction on $\St \equiv \St(d, k)$ is a smooth mapping $\retr: \Tg\St \rightarrow \St$ from the tangent bundle $\Tg\St$ onto $\St$ such that the restriction of $\retr$ onto $\Tg_Z\St$, denoted by $\retr_Z$, satisfies that (i) $\retr_Z(0) = Z$ for all $Z \in \St$ where $0$ denotes the zero element of $\textnormal{T}\St$, and (ii) for any $Z \in \St$, it holds that $\lim_{\xi \in \Tg_Z\St, \xi \rightarrow 0} \|\retr_Z(\xi) - (Z + \xi)\|_F/\|\xi\|_F = 0$.
\end{definition}
The retraction on the Stiefel manifold has the following well-known properties~\citep{Boumal-2019-Global, Liu-2019-Quadratic} which are important to subsequent analysis in this paper.  
\begin{proposition}\label{prop:retraction}
For all $Z \in \St \equiv \St(d, k)$ and $\xi \in \Tg_Z\St$, there exist constants $L_1 > 0$ and $L_2 > 0$ such that the following two inequalities hold: 
\begin{eqnarray*}
\|\retr_Z(\xi) - Z\|_F & \leq & L_1\|\xi\|_F, \\
\|\retr_Z(\xi) - (Z + \xi)\|_F & \leq & L_2\|\xi\|_F^2. 
\end{eqnarray*}
\end{proposition}
For the sake of completeness, we provide four popular restrictions~\citep{Edelman-1998-Geometry, Wen-2013-Feasible, Liu-2019-Quadratic, Chen-2020-Proximal} on the Stiefel manifold in practice. Determining which one is the most efficient in the algorithm is still an open question; see the discussion after~\citet[Theorem~3]{Liu-2019-Quadratic} or before~\citet[Fact~3.6]{Chen-2020-Proximal}. 
\begin{itemize}
\item \textbf{Exponential mapping.} It takes $8dk^2 + \bigO(k^3)$ flops and has the closed-form expression:  
\begin{equation*}
\retr_Z^{\exp}(\xi) \ = \ \begin{bmatrix} Z & Q\end{bmatrix}\exp\left(\begin{bmatrix} -Z^\top\xi & -R^\top \\ R & 0 \end{bmatrix}\right)\begin{bmatrix} I_k \\ 0 \end{bmatrix}.
\end{equation*}
where $QR = -(I_k - ZZ^\top)\xi$ is the unique QR factorization. 
\item \textbf{Polar decomposition.} It takes $3dk^2 + \bigO(k^3)$ flops and has the closed-form expression: 
\begin{equation*}
\retr_Z^{\textnormal{polar}}(\xi) \ = \ (Z + \xi)(I_k + \xi^\top\xi)^{-1/2}.
\end{equation*} 
\item \textbf{QR decomposition.} It takes $2dk^2 + \bigO(k^3)$ flops and has the closed-form expression: 
\begin{equation*}
\retr_Z^{\textnormal{qr}}(\xi) \ = \ \textnormal{qr}(Z + \xi), 
\end{equation*} 
where $\textnormal{qr}(A)$ is the Q factor of the QR factorization of $A$.
\item \textbf{Cayley transformation.} It takes $7dk^2 + \bigO(k^3)$ flops and has the closed-form expression:
\begin{equation*}
\retr_Z^{\textnormal{cayley}}(\xi) \ = \ \left(I_n - \frac{1}{2}W(\xi) \right)^{-1}\left(I_n + \frac{1}{2}W(\xi) \right)Z,
\end{equation*}
where $W(\xi) = (I_n - ZZ^\top/2)\xi Z^\top - Z\xi^\top(I_n - ZZ^\top/2)$.
\end{itemize}
We now present a novel approach to exploiting the structure of $f$. We begin with several definitions.
\begin{algorithm}[!t]
\caption{Riemannian Adaptive Gradient Ascent with Sinkhorn Iteration (RAGAS)}\label{alg:adagrad-sinkhorn}
\begin{algorithmic}[1]
\STATE \textbf{Input:} $\{(x_i, r_i)\}_{i \in [n]}$ and $\{(y_j, c_j)\}_{j \in [n]}$, $k = \bigOtil(1)$, $U_0 \in \St(d, k)$, $\epsilon$ and $\alpha \in (0, 1)$. 
\STATE \textbf{Initialize:} $p_0 = \zero_d$, $q_0 = \zero_k$, $\widehat{p_0} = \alpha\|C\|_\infty^2\one_d$, $\widehat{q_0} = \alpha\|C\|_\infty^2\one_k$, $\widehat{\epsilon} \leftarrow \frac{\epsilon\sqrt{\alpha}}{20\|C\|_\infty}$, $\eta \leftarrow \frac{\epsilon\min\{1, 1/\bar{\theta}\}}{40\log(n)}$ and $\gamma \leftarrow \frac{\alpha}{16L_1^2 + 32L_2 + 32\eta^{-1}L_1^2\|C\|_\infty}$. 
\FOR{$t = 0, 1, 2, \ldots$}
\STATE Compute $\pi_{t+1} \leftarrow \textsc{regOT}(\{(x_i, r_i)\}_{i \in [n]}, \{(y_j, c_j)\}_{j \in [n]}, U_t, \eta, \widehat{\epsilon})$.
\STATE Compute $G_{t+1} \leftarrow P_{\Tg_{U_t}\St}(2V_{\pi_{t+1}}U_t)$. 
\STATE Update $p_{t+1} \leftarrow \beta p_t + (1-\beta)\diag(G_{t+1}G_{t+1}^\top)/k$ and $\widehat{p}_{t+1} \leftarrow \max\{\widehat{p}_t, p_{t+1}\}$.
\STATE Update $q_{t+1} \leftarrow \beta q_t + (1-\beta)\diag(G_{t+1}^\top G_{t+1})/d$ and $\widehat{q}_{t+1} \leftarrow \max\{\widehat{q}_t, q_{t+1}\}$. 
\STATE Compute $\xi_{t+1} \leftarrow P_{\Tg_{U_t}\St}(\Diag(\widehat{p}_{t+1})^{-1/4}G_{t+1}\Diag(\widehat{q}_{t+1})^{-1/4})$. 
\STATE Compute $U_{t+1} \leftarrow \retr_{U_t}(\gamma\xi_{t+1})$.  
\ENDFOR
\end{algorithmic}
\end{algorithm}
\begin{definition}\label{def:C}
The \emph{coefficient matrix} between $\mu = \sum_{i=1}^n r_i\delta_{x_i}$ and $\nu = \sum_{j=1}^n c_j\delta_{y_j}$ is defined by $C = (C_{ij})_{1 \leq i, j \leq n} \in \br^{n \times n}$ with each entry $C_{ij} = \|x_i - y_j\|^2$. 
\end{definition}
\begin{definition}\label{def:V}
The \emph{correlation matrix} between $\mu = \sum_{i=1}^n r_i\delta_{x_i}$ and $\nu = \sum_{j=1}^n c_j\delta_{y_j}$ is defined by $V_\pi = \sum_{i=1}^n \sum_{j=1}^n \pi_{i, j} (x_i - y_j)(x_i - y_j)^\top \in \br^{d \times d}$. 
\end{definition}
The first lemma shows that the structure of the function $f$ is not very bad regardless of nonconvexity and the lack of smoothness. 
\begin{lemma}\label{lemma:obj-weak-concave}
The function $f$ is $2\|C\|_\infty$-weakly concave. 
\end{lemma}
\begin{proof}
By~\citet[Proposition~4.3]{Vial-1983-Strong}, it suffices to show that the function $f(U) - \|C\|_\infty\|U\|_F^2$ is concave for any $U \in \br^{d \times k}$. By the definition of $f$, we have
\begin{equation*}
f(U) \ = \ \min\limits_{\pi \in \Pi(\mu, \nu)} \trac\left(U^\top V_\pi U\right). 
\end{equation*}
Since $\{x_1, x_2, \ldots, x_n\} \subseteq \br^d$ and $\{y_1, y_2, \ldots, y_n\} \subseteq \br^d$ are two given groups of $n$ atoms in $\br^d$, the coefficient matrix $C$ is independent of $U$ and $\pi$. Furthermore, $\sum_{i=1}^n \sum_{j=1}^n \pi_{i, j} = 1$ and $\pi_{i, j} \geq 0$ for all $i, j \in [n]$ since $\pi \in \Pi(\mu, \nu)$. Putting these pieces together with Jensen's inequality, we have
\begin{equation*}
\|V_\pi\|_F \ \leq \ \sum_{i=1}^n \sum_{j=1}^n \pi_{i, j} \|(x_i - y_j)(x_i - y_j)^\top\|_F \ \leq \ \max_{1 \leq i, j \leq n} \|x_i - y_j\|^2 \ = \ \|C\|_\infty. 
\end{equation*}
This implies that $U \mapsto \trac(U^\top V_\pi U) - \|C\|_\infty\|U\|_F^2$ is concave for any $\pi \in \Pi(\mu, \nu)$. Since $\Pi(\mu, \nu)$ is compact, Danskin's theorem~\citep{Rockafellar-2015-Convex} implies the desired result. 
\end{proof}
The second lemma shows that the subdifferential of the function $f$ is independent of $U$ and bounded by a constant $2\|C\|_\infty$. 
\begin{lemma}\label{lemma:bound-subdiff}
Each element of the subdifferential $\partial f(U)$ is bounded by $2\|C\|_\infty$ for all $U \in \St(d, k)$. 
\end{lemma}
\begin{proof}
By the definition of the subdifferential $\partial f$, it suffices to show that $\|V_\pi U \|_F \leq \|C\|_\infty$ for all $\pi \in \Pi(\mu, \nu)$ and $U \in \St(d, k)$. Indeed, by the definition, $V_\pi$ is symmetric and positive semi-definite. Therefore, we have
\begin{equation*}
\max_{U \in \St(d, k)} \|V_\pi U \|_F \ \leq \ \|V_\pi\|_F \ \leq \ \|C\|_\infty. 
\end{equation*}
Putting these pieces together yields the desired result. 
\end{proof}
\begin{remark}
Lemma~\ref{lemma:obj-weak-concave} implies there exists a concave function $g: \br^{d \times k} \rightarrow \br$ such that $f(U) = g(U) + \|C\|_\infty\|U\|_F^2$ for any $U \in \br^{d \times k}$. Since $g$ is concave, $\partial g$ is well defined and~\citet[Proposition~4.6]{Vial-1983-Strong} implies that $\partial f(U) = \partial g(U) + 2\|C\|_\infty U$ for all $U \in \br^{d \times k}$.
\end{remark}
This result together with~\citet[Proposition~4.5]{Vial-1983-Strong} and~\citet[Theorem~5.1]{Yang-2014-Optimality} lead to the Riemannian subdifferential defined by $\subdiff f(U) = P_{\Tg_U\St}(\partial f(U))$ for all $U \in \St(d, k)$. 
\begin{definition}\label{def:stationarity}
The subspace projection $\widehat{U} \in \St(d, k)$ is called an \emph{$\epsilon$-approximate optimal subspace projection} of $f$ over $\St(d, k)$ in Eq.~\eqref{prob:Stiefel-nonsmooth} if it satisfies  $\dist(0, \subdiff f(\widehat{U})) \leq \epsilon$. 
\end{definition}
\begin{definition}\label{def:optimal-pair}
The pair of subspace projection and transportation plan $(\widehat{U}, \widehat{\pi}) \in \St(d, k) \times \Pi(\mu, \nu)$ is an \emph{$\epsilon$-approximate pair of optimal subspace projection and optimal transportation plan} for the computation of the PRW distance in Eq.~\eqref{prob:main} if the following statements hold true: (i) $\widehat{U}$ is an $\epsilon$-approximate optimal subspace projection of $f$ over $\St(d, k)$ in Eq.~\eqref{prob:Stiefel-nonsmooth}. (ii) $\widehat{\pi}$ is an $\epsilon$-approximate optimal transportation plan for the subspace projection $\widehat{U}$.
\end{definition}
The goal of this paper is to develop a set of algorithms which are guaranteed to converge to a pair of approximate optimal subspace projection and optimal transportation plan, which stand for a stationary point of the max-min optimization model in Eq.~\eqref{prob:main}. In the next section, we provide the detailed scheme of our algorithm as well as the finite-time theoretical guarantee.

\section{Riemannian (Adaptive) Gradient meets Sinkhorn Iteration}\label{sec:algorithm}
We present the \emph{Riemannian gradient ascent with Sinkhorn} (RGAS) algorithm for solving Eq.~\eqref{prob:Stiefel-smooth}. By the definition of $V_\pi$ (cf.\ Definition~\ref{def:V}), we can rewrite $f_\eta(U) = \min_{\pi \in \Pi(\mu, \nu)} \{\langle UU^\top, V_\pi\rangle - \eta H(\pi)\}$. Fix $U \in \br^{d \times k}$, and define the mapping $\pi \mapsto \langle UU^\top, V_\pi\rangle - \eta H(\pi)$ with respect to $\ell_1$-norm. By the compactness of the transportation polytope $\Pi(\mu, \nu)$, Danskin's theorem~\citep{Rockafellar-2015-Convex} implies that $f_\eta$ is smooth. Moreover, by the symmetry of $V_\pi$, we have
\begin{equation}\label{def:grad-entropy-regularization}
\nabla f_\eta(U) \ = \ 2V_{\pi^\star(U)} U \quad \textnormal{for any } U \in \br^{d \times k},
\end{equation}
where $\pi^\star(U) \mydefn \argmin_{\pi \in \Pi(\mu, \nu)} \ \{\langle UU^\top, V_\pi\rangle - \eta H(\pi)\}$. This entropic regularized OT is solved inexactly at each inner loop of the maximization and we use the output $\pi_{t+1} \approx \pi(U_t)$ to obtain an inexact gradient of $f_\eta$ which permits the Riemannian gradient ascent update; see Algorithm~\ref{alg:grad-sinkhorn}. Note that the stopping criterion used here is set as $\|\pi_{t+1} - \pi(U_t)\|_1 \leq \widehat{\epsilon}$ which implies that $\pi_{t+1}$ is $\epsilon$-approximate optimal transport plan for $U_t \in \St(d, k)$. 

The remaining issue is to approximately solve an entropic regularized OT efficiently. We leverage Cuturi's approach and obtain the desired output $\pi_{t+1}$ for $U_t \in \St(d, k)$ using the Sinkhorn iteration. By adapting the proof presented by~\citet[Theorem~1]{Dvurechensky-2018-Computational}, we derive that Sinkhorn iteration achieves a finite-time guarantee which is polynomial in $n$ and $1/\widehat{\epsilon}$. As a practical enhancement, we develop the \emph{Riemannian adaptive gradient ascent with Sinkhorn} (RAGAS) algorithm by exploiting the matrix structure of $\grad f_\eta(U_t)$ via the use of two different adaptive weight vectors, namely $\widehat{p}_t$ and $\widehat{q}_t$; see the adaptive algorithm in Algorithm~\ref{alg:adagrad-sinkhorn}. It is worth mentioning that such an adaptive strategy is proposed by~\citet{Kasai-2019-Riemannian} and has been shown to generate a search direction which is better than the Riemannian gradient $\grad f_\eta(U_t)$ in terms of robustness to the stepsize. 

\subsection{Technical lemmas}
We first show that $f_\eta$ is continuously differentiable over $\br^{d \times k}$ and the classical gradient inequality holds true over $\St(d, k)$. The derivation is novel and uncovers the structure of the computation of entropic regularized PRW in Eq.~\eqref{prob:main-regularized}. Let $g: \br^{d \times k} \times \Pi(\mu, \nu) \rightarrow \br$ be defined by 
\begin{equation*}
g(U, \pi) \ : = \ \sum_{i=1}^n \sum_{j=1}^n \pi_{i, j} \|U^\top x_i - U^\top y_j\|^2 - \eta H(\pi).
\end{equation*}
\begin{lemma}\label{lemma:lip-grad}
$f_\eta$ is differentiable over $\br^{d \times k}$ and $\|\nabla f_\eta(U)\|_F \leq 2\|C\|_\infty$ for all $U \in \St(d, k)$. 
\end{lemma}
\begin{proof}
It is clear that we have $f_\eta(\bullet) = \min_{\pi \in \Pi(\mu, \nu)} g(\bullet, \pi)$. Furthermore, $\pi^\star(\bullet) = \argmin_{\pi \in \Pi(\mu, \nu)} g(\bullet, \pi)$ is uniquely defined. Putting these pieces with the compactness of $\Pi(\mu, \nu)$ and the smoothness of $g(\bullet, \pi)$, Danskin's theorem~\citep{Rockafellar-2015-Convex} implies $f_\eta$ is continuously differentiable and the gradient is
\begin{equation*}
\nabla f_\eta(U) \ = \ 2V_{\pi^\star(U)}U \quad \text{for all } U \in \br^{d \times k}. 
\end{equation*}
Since $U \in \St(d, k)$ and $\pi^\star(U) \in \Pi(\mu, \nu)$, we have 
\begin{equation*}
\|\nabla f_\eta(U)\|_F \ = \ 2\|V_{\pi^\star(U)}U\|_F \ \leq \ 2\|V_{\pi^\star(U)}\|_F \ \leq \ 2\|C\|_\infty. 
\end{equation*}
This completes the proof. 
\end{proof}
\begin{lemma}\label{lemma:key-inequality}
For all $U_1, U_2 \in \St(d, k)$, the following statement holds true, 
\begin{equation*}
|f_\eta(U_1) - f_\eta(U_2) - \langle \nabla f_\eta(U_2), U_1 - U_2\rangle| \ \leq \ \left( \|C\|_\infty + \frac{2 \|C\|_\infty^2}{\eta}\right)\|U_1 - U_2\|_F^2.   
\end{equation*}
\end{lemma}
\begin{proof}
It suffices to prove that
\begin{equation*}
\|\nabla f_\eta(\alpha U_1 + (1-\alpha)U_2) - \nabla f_\eta(U_2)\|_F \ \leq \ \left(2 \|C\|_\infty + \frac{4\|C\|_\infty^2}{\eta}\right) \alpha \|U_1 - U_2\|_F,
\end{equation*}
for any $U_1, U_2 \in \St(d, k)$ and any $\alpha \in [0, 1]$. Indeed, let $U_\alpha = \alpha U_1 + (1-\alpha)U_2$, we have
\begin{equation*}
\|\nabla f_\eta(U_\alpha) - \nabla f_\eta(U_2)\|_F \ \leq \ 2\|V_{\pi^\star(U_\alpha)}\|_F\|U_\alpha - U_2\|_F + 2\|V_{\pi^\star(U_\alpha)} - V_{\pi^\star(U_2)}\|_F. 
\end{equation*}
Since $\pi^\star(U_\alpha) \in \Pi(\mu, \nu)$, we have $\|V_{\pi^\star(U_\alpha)}\|_F \leq \|C\|_\infty$. By the definition of $V_\pi$, we have
\begin{equation*}
\|V_{\pi^\star(U_\alpha)} - V_{\pi^\star(U_2)}\|_F \ \leq \ \sum_{i=1}^n \sum_{j=1}^n |\pi_{i, j}^\star(U_\alpha) - \pi_{i, j}^\star(U_2)|\|x_i - y_j\|^2 \ \leq \ \|C\|_\infty\|\pi^\star(U_\alpha) - \pi^\star(U_2)\|_1. 
\end{equation*}
Putting these pieces together yields that 
\begin{equation}\label{inequality-lip-grad-first}
\|\nabla f_\eta(U_\alpha) - \nabla f_\eta(U_2)\|_F \ \leq \ 2\|C\|_\infty\|U_\alpha - U_2\|_F + 2\|C\|_\infty\|\pi^\star(U_\alpha) - \pi^\star(U_2)\|_1. 
\end{equation}
Using the property of the entropy regularization $H(\bullet)$, we have $g(U, \bullet)$ is strongly convex with respect to $\ell_1$-norm and the module is $\eta$. This implies that 
\begin{eqnarray*}
g(U_\alpha, \pi^\star(U_2)) - g(U_\alpha, \pi^\star(U_\alpha)) - \langle\nabla_\pi g(U_\alpha, \pi^\star(U_\alpha)), \pi^\star(U_2) - \pi^\star(U_\alpha)\rangle & \geq & (\eta/2)\|\pi^\star(U_\alpha) - \pi^\star(U_2)\|_1^2, \\
g(U_\alpha, \pi^\star(U_\alpha)) -  g(U_\alpha, \pi^\star(U_2)) - \langle\nabla_\pi g(U_\alpha, \pi^\star(U_2)), \pi^\star(U_\alpha) - \pi^\star(U_2)\rangle & \geq & (\eta/2)\|\pi^\star(U_\alpha) - \pi^\star(U_2)\|_1^2. 
\end{eqnarray*}
Summing up these inequalities yields 
\begin{equation}\label{inequality-lip-grad-second}
\langle \nabla_\pi g(U_\alpha, \pi^\star(U_\alpha)) - \nabla_\pi g(U_\alpha, \pi^\star(U_2)), \pi^\star(U_\alpha) - \pi^\star(U_2)\rangle \ \geq \ \eta\|\pi^\star(U_\alpha) - \pi^\star(U_2)\|_1^2. 
\end{equation}
Furthermore, by the first-order optimality condition of $\pi^\star(U_1)$ and $\pi^\star(U_2)$, we have
\begin{eqnarray*}
\langle\nabla_\pi g(U_\alpha, \pi^\star(U_\alpha)), \pi^\star(U_2) - \pi^\star(U_\alpha)\rangle & \geq & 0, \\
\langle\nabla_\pi g(U_2, \pi^\star(U_2)), \pi^\star(U_\alpha) - \pi^\star(U_2)\rangle & \geq & 0. 
\end{eqnarray*}
Summing up these inequalities yields 
\begin{equation}\label{inequality-lip-grad-third}
\langle \nabla_\pi g(U_2, \pi^\star(U_2)) - \nabla_\pi g(U_\alpha, \pi^\star(U_\alpha)), \pi^\star(U_\alpha) - \pi^\star(U_2)\rangle \ \geq \ 0. 
\end{equation}
Summing up Eq.~\eqref{inequality-lip-grad-second} and Eq.~\eqref{inequality-lip-grad-third} and further using H\"{o}lder's inequality, we have
\begin{equation*}
\|\pi^\star(U_\alpha) - \pi^\star(U_2)\|_1 \ \leq \ (1/\eta)\|\nabla_\pi g(U_2, \pi^\star(U_2)) - \nabla_\pi g(U_\alpha, \pi^\star(U_2))\|_\infty. 
\end{equation*}
By the definition of function $g$, we have
\begin{eqnarray*}
\|\nabla_\pi g(U_2, \pi^\star(U_2)) - \nabla_\pi g(U_\alpha, \pi^\star(U_2))\|_\infty & \leq & \max_{1 \leq i, j \leq n} |(x_i - x_j)^\top(U_2U_2^\top - U_\alpha U_\alpha^\top)(x_i - x_j)| \\ 
& & \hspace*{-8em} \leq \ \left(\max_{1 \leq i, j \leq n} \| x_i - y_j\|^2\right)\|U_2U_2^\top - U_\alpha U_\alpha^\top\|_F \\
& & \hspace*{-8em} = \ \|C\|_\infty\|U_2 U_2^\top - U_\alpha U_\alpha^\top\|_F. 
\end{eqnarray*}
Since $U_1, U_2 \in \St(d, k)$, we have
\begin{eqnarray*}
\|U_2U_2^\top - U_\alpha U_\alpha^\top\|_F & \leq & \|U_2(U_2 - U_\alpha)^\top\|_F + \|(U_2 - U_\alpha)U_\alpha^\top\|_F \\ 
& \leq & \|U_2 - U_\alpha\|_F + \|(U_2 - U_\alpha)(\alpha U_1 + (1-\alpha)U_2)^\top\|_F \\
& \leq & \|U_2 - U_\alpha\|_F + \alpha\|(U_2 - U_\alpha)U_1^\top\|_F + (1-\alpha)\|(U_2 - U_\alpha)U_2^\top\|_F \\
& \leq & 2\|U_2 - U_\alpha\|_F. 
\end{eqnarray*}
Putting these pieces together yields that 
\begin{equation}\label{inequality-lip-grad-fourth}
\|\pi^\star(U_\alpha) - \pi^\star(U_2)\|_1 \ \leq \ \frac{2 \|C\|_\infty}{\eta} \|U_\alpha - U_2\|_F. 
\end{equation}
Plugging Eq.~\eqref{inequality-lip-grad-fourth} into Eq.~\eqref{inequality-lip-grad-first} yields the desired result. 
\end{proof}
\begin{remark}
Lemma~\ref{lemma:key-inequality} shows that $f_\eta$ satisfies the classical gradient inequality over the Stiefel manifold. This is indeed stronger than the following statement, 
\begin{equation*}
\|\nabla f_\eta(U_1) - \nabla f_\eta(U_2)\|_F \ \leq \ \left(2 \|C\|_\infty + \frac{4 \|C\|_\infty^2}{\eta}\right)\|U_1 - U_2\|_F, \quad \textnormal{for all } U_1, U_2 \in \St(d, k),  
\end{equation*}
and forms the basis for analyzing the complexity bound of Algorithm~\ref{alg:grad-sinkhorn} and~\ref{alg:adagrad-sinkhorn}. The techniques used in proving Lemma~\ref{lemma:key-inequality} are new and may be applicable to analyze the structure of the robust variant of the Wasserstein distance with other type of regularization~\citep{Dessein-2018-Regularized, Blondel-2018-Smooth}. 
\end{remark}
Then we quantify the progress of RGAS algorithm (cf. Algorithm~\ref{alg:grad-sinkhorn}) using $f_\eta$ as a potential function and then provide an upper bound for the number of iterations to return an $\epsilon$-approximate optimal subspace projection $U_t \in \St(d, k)$ satisfying $\dist(0, \subdiff f(U_t)) \leq \epsilon$ in Algorithm~\ref{alg:grad-sinkhorn}. 
\begin{lemma}\label{lemma:obj-progress-grad}
Let $\{(U_t, \pi_t)\}_{t \geq 1}$ be the iterates generated by Algorithm~\ref{alg:grad-sinkhorn}. We have
\begin{equation*}
\frac{1}{T}\left(\sum_{t=0}^{T-1} \|\grad f_\eta(U_t)\|_F^2\right) \ \leq \ \frac{4\Delta_f}{\gamma T} + \frac{\epsilon^2}{5}, 
\end{equation*}
where $\Delta_f = \max_{U \in \St(d, k)} f_\eta(U) - f_\eta(U_0)$ is the initial objective gap. 
\end{lemma}
\begin{proof}
Using Lemma~\ref{lemma:key-inequality} with $U_1 = U_{t+1}$ and $U_2 = U_t$, we have
\begin{equation}\label{inequality-descent-grad-first}
f_\eta(U_{t+1}) - f_\eta(U_t) - \langle \nabla f_\eta(U_t), U_{t+1} - U_t\rangle \ \geq \ -\left(\|C\|_\infty + \frac{2 \|C\|_\infty^2}{\eta}\right)\|U_{t+1} - U_t\|_F^2.   
\end{equation}
By the definition of $U_{t+1}$, we have
\begin{eqnarray*}
\langle \nabla f_\eta(U_t), U_{t+1} - U_t\rangle & = & \ \langle \nabla f_\eta(U_t), \retr_{U_t}(\gamma\xi_{t+1}) - U_t\rangle \\ 
& & \hspace*{-6em} = \ \langle\nabla f_\eta(U_t), \gamma\xi_{t+1}\rangle + \langle\nabla f_\eta(U_t), \retr_{U_t}(\gamma\xi_{t+1}) - (U_t + \gamma\xi_{t+1})\rangle \\
& & \hspace*{-6em} \geq \ \langle\nabla f_\eta(U_t), \gamma\xi_{t+1}\rangle - \|\nabla f_\eta(U_t)\|_F\|\retr_{U_t}(\gamma\xi_{t+1}) - (U_t + \gamma\xi_{t+1})\|_F. 
\end{eqnarray*}
By Lemma~\ref{lemma:lip-grad}, we have $\|\nabla f_\eta(U)\|_F \leq 2\|C\|_\infty$. Putting these pieces with Proposition~\ref{prop:retraction} yields that 
\begin{equation}\label{inequality-descent-grad-second}
\langle \nabla f_\eta(U_t), U_{t+1} - U_t\rangle \ \geq \ \gamma\langle\nabla f_\eta(U_t), \xi_{t+1}\rangle - 2\gamma^2 L_2\|C\|_\infty\|\xi_{t+1}\|_F^2. 
\end{equation}
Using Proposition~\ref{prop:retraction} again, we have
\begin{equation}\label{inequality-descent-grad-third}
\|U_{t+1} - U_t\|_F^2 \ = \ \|\retr_{U_t}(\gamma\xi_{t+1}) - U_t\|_F^2 \ \leq \ \gamma^2 L_1^2\|\xi_{t+1}\|_F^2. 
\end{equation}
Combining Eq.~\eqref{inequality-descent-grad-first}, Eq.~\eqref{inequality-descent-grad-second} and Eq.~\eqref{inequality-descent-grad-third} yields  
\begin{equation}\label{inequality-descent-grad-fourth}
f_\eta(U_{t+1}) - f_\eta(U_t) \ \geq \ \gamma\langle\nabla f_\eta(U_t), \xi_{t+1}\rangle - \gamma^2((L_1^2 + 2 L_2)\|C\|_\infty + 2 \eta^{-1}L_1^2\|C\|_\infty^2)\|\xi_{t+1}\|_F^2. 
\end{equation}
Recall that $\grad f_\eta(U_t) = P_{\Tg_{U_t}\St}(\nabla f_\eta(U_t))$ and $\xi_{t+1} = P_{\Tg_{U_t}\St}(2V_{\pi_{t+1}}U_t)$, we have 
\begin{equation*}
\langle\nabla f_\eta(U_t), \xi_{t+1}\rangle \ = \ \langle \grad f_\eta(U_t), \xi_{t+1}\rangle \ = \ \|\grad f_\eta(U_t)\|_F^2 + \langle\grad f_\eta(U_t), \xi_{t+1} - \grad f_\eta(U_t)\rangle
\end{equation*}
Using Young's inequality, we have
\begin{equation*}
\langle\nabla f_\eta(U_t), \xi_{t+1}\rangle \ \geq \ (1/2)\left(\|\grad f_\eta(U_t)\|_F^2 - \|\xi_{t+1} - \grad f_\eta(U_t)\|_F^2\right). 
\end{equation*}
Furthermore, we have $\|\xi_{t+1}\|_F^2 \leq 2\|\grad f_\eta(U_t)\|_F^2 + 2\|\xi_{t+1} - \grad f_\eta(U_t)\|_F^2$. Putting these pieces together with Eq.~\eqref{inequality-descent-grad-fourth} yields that 
\begin{eqnarray}\label{inequality-descent-grad-fifth}
f_\eta(U_{t+1}) - f_\eta(U_t) & \geq & \gamma\left(\frac{1}{2} - \gamma(2 L_1^2\|C\|_\infty + 4L_2\|C\|_\infty + 4\eta^{-1}L_1^2\|C\|_\infty^2)\right)\|\grad f_\eta(U_t)\|_F^2 \nonumber \\
& & \hspace*{-6em} - \gamma\left(\frac{1}{2} + \gamma(2 L_1^2\|C\|_\infty + 4L_2\|C\|_\infty + 4 \eta^{-1}L_1^2\|C\|_\infty^2)\right)\|\xi_{t+1} - \grad f_\eta(U_t)\|_F^2. 
\end{eqnarray}
Since $\xi_{t+1} = P_{\Tg_{U_t}\St}(2V_{\pi_{t+1}}U_t)$ and $\grad f_\eta(U_t) = P_{\Tg_{U_t}\St}(2V_{\tilde{\pi}_t^\star}U_t)$ where $\tilde{\pi}_t^\star$ is a minimizer of the entropic regularized OT problem, i.e., $\tilde{\pi}_t^\star \in \argmin_{\pi \in \Pi(\mu, \nu)} \ \{\langle U_tU_t^\top, V_\pi\rangle - \eta H(\pi)\}$, we have
\begin{equation*}
\|\xi_{t+1} - \grad f_\eta(U_t)\|_F \ \leq \ 2\|(V_{\pi_{t+1}} - V_{\tilde{\pi}_t^\star})U_t\|_F \ = \ 2\|V_{\pi_{t+1}} - V_{\tilde{\pi}_t^\star}\|_F. 
\end{equation*}
By the definition of $V_\pi$ and using the stopping criterion: $\|\pi_{t+1} - \tilde{\pi}_t^\star\|_1 \leq \widehat{\epsilon} = \frac{\epsilon}{10\|C\|_\infty}$, we have
\begin{equation*}
\|V_{\pi_{t+1}} - V_{\tilde{\pi}_t^\star}\|_F \ \leq \ \|C\|_\infty\|\pi_{t+1} - \tilde{\pi}_t^\star\|_1 \leq \frac{\epsilon}{10}. 
\end{equation*}
Putting these pieces together yields that 
\begin{equation}\label{inequality-descent-grad-sixth}
\|\xi_{t+1} - \grad f_\eta(U_t)\|_F \ \leq \ \frac{\epsilon}{5}. 
\end{equation}
Plugging Eq.~\eqref{inequality-descent-grad-sixth} into Eq.~\eqref{inequality-descent-grad-fifth} with the definition of $\gamma$ yields that
\begin{equation*}
f_\eta(U_{t+1}) - f_\eta(U_t) \ \geq \ \frac{\gamma\|\grad f_\eta(U_t)\|_F^2}{4} - 
\frac{\gamma\epsilon^2}{20}.
\end{equation*}
Summing and rearranging the resulting inequality yields that 
\begin{equation*}
\frac{1}{T}\left(\sum_{t=0}^{T-1} \|\grad f_\eta(U_t)\|_F^2\right) \ \leq \ \frac{4(f_\eta(U_T) - f_\eta(U_0))}{\gamma T} + \frac{\epsilon^2}{5}. 
\end{equation*}
This together with the definition of $\Delta_f$ implies the desired result. 
\end{proof}
We now provide analogous results for the RAGAS algorithm (cf.\ Algorithm~\ref{alg:adagrad-sinkhorn}). 
\begin{lemma}\label{lemma:obj-progress-adagrad}
Let $\{(U_t, \pi_t)\}_{t \geq 1}$ be the iterates generated by Algorithm~\ref{alg:adagrad-sinkhorn}. Then, we have
\begin{equation*}
\frac{1}{T}\left(\sum_{t=0}^{T-1} \|\grad f_\eta(U_t)\|_F^2\right) \ \leq \ \frac{8\|C\|_\infty\Delta_f}{\gamma T} + \frac{\epsilon^2}{10}, 
\end{equation*}
where $\Delta_f = \max_{U \in \St(d, k)} f_\eta(U) - f_\eta(U_0)$ is the initial objective gap. 
\end{lemma}
\begin{proof}
Using the same argument as in the proof of Lemma~\ref{lemma:obj-progress-grad}, we have
\begin{equation}\label{inequality-descent-adagrad-first}
f_\eta(U_{t+1}) - f_\eta(U_t) \ \geq \ \gamma\langle\nabla f_\eta(U_t), \xi_{t+1}\rangle - \gamma^2((L_1^2 + 2 L_2)\|C\|_\infty + 2 \eta^{-1}L_1^2\|C\|_\infty^2)\|\xi_{t+1}\|_F^2. 
\end{equation}
Recall that $\grad f_\eta(U_t) = P_{\Tg_{U_t}\St}(\nabla f_\eta(U_t))$ and the definition of $\xi_{t+1}$, we have
\begin{eqnarray*}
\langle\nabla f_\eta(U_t), \xi_{t+1}\rangle & = & \langle \grad f_\eta(U_t), \xi_{t+1}\rangle \\
& = & \langle \grad f_\eta(U_t), \Diag(\widehat{p}_{t+1})^{-1/4}(\grad f_\eta(U_t))\Diag(\widehat{q}_{t+1})^{-1/4}\rangle \\
& & \hspace*{2 em} + \langle\grad f_\eta(U_t), \Diag(\widehat{p}_{t+1})^{-1/4}(G_{t+1} - \grad f_\eta(U_t))\Diag(\widehat{q}_{t+1})^{-1/4}\rangle.
\end{eqnarray*}
Using the Cauchy-Schwarz inequality and the nonexpansiveness of $P_{\Tg_{U_t}\St}$, we have
\begin{eqnarray*}
\|\xi_{t+1}\|_F^2 & \leq & 2\|P_{\Tg_{U_t}\St}(\Diag(\widehat{p}_{t+1})^{-1/4}(\grad f_\eta(U_t))\Diag(\widehat{q}_{t+1})^{-1/4})\|_F^2 \\ 
& & + 2\|\xi_{t+1} - P_{\Tg_{U_t}\St}(\Diag(\widehat{p}_{t+1})^{-1/4}(\grad f_\eta(U_t))\Diag(\widehat{q}_{t+1})^{-1/4})\|_F^2 \\ 
& \leq & 2\|\Diag(\widehat{p}_{t+1})^{-1/4}(\grad f_\eta(U_t))\Diag(\widehat{q}_{t+1})^{-1/4}\|_F^2 \\ 
& & + 2\|\Diag(\widehat{p}_{t+1})^{-1/4}(G_{t+1} - \grad f_\eta(U_t))\Diag(\widehat{q}_{t+1})^{-1/4}\|_F^2. 
\end{eqnarray*}
Furthermore, by the definition of $G_{t+1}$, we have $\|G_{t+1}\|_F \leq 2\|C\|_\infty$ and hence 
\begin{equation*}
\zero_d \leq \frac{\diag(G_{t+1}G_{t+1}^\top)}{k} \leq 4\|C\|_\infty^2\one_d, \qquad \zero_k \leq \frac{\diag(G_{t+1}^\top G_{t+1})}{d} \preceq 4\|C\|_\infty^2\one_k. 
\end{equation*} 
By the definition of $p_t$ and $q_t$, we have $\zero_d \preceq p_t \preceq 4\|C\|_\infty^2\one_d$ and $\zero_k \preceq q_t \preceq 4\|C\|_\infty^2\one_k$. This together with the definition of $\widehat{p}_t$ and $\widehat{q}_t$ implies that 
\begin{equation*}
\alpha\|C\|_\infty^2 \one_d \leq \widehat{p}_t \leq 4\|C\|_\infty^2\one_d, \qquad \alpha\|C\|_\infty^2 \one_k \leq \widehat{q}_t \leq 4\|C\|_\infty^2\one_k. 
\end{equation*}
This inequality together with Young's inequality implies that 
\begin{eqnarray*}
\langle\nabla f_\eta(U_t), \xi_{t+1}\rangle & \geq & \frac{\|\grad f_\eta(U_t)\|_F^2}{2\|C\|_\infty} - \frac{1}{\sqrt{\alpha}\|C\|_\infty}\left(\frac{\sqrt{\alpha}\|\grad f_\eta(U_t)\|_F^2}{4} + \frac{\|G_{t+1} - \grad f_\eta(U_t)\|_F^2}{\sqrt{\alpha}} \right) \\
& = & \frac{\|\grad f_\eta(U_t)\|_F^2}{4\|C\|_\infty} - \frac{\|G_{t+1} - \grad f_\eta(U_t)\|_F^2}{\alpha\|C\|_\infty}, 
\end{eqnarray*}
and 
\begin{equation*}
\|\xi_{t+1}\|_F^2 \ \leq \ \frac{2\|\grad f_\eta(U_t)\|_F^2}{\alpha\|C\|_\infty^2} + \frac{2\|G_{t+1} - \grad f_\eta(U_t)\|_F^2}{\alpha\|C\|_\infty^2}.
\end{equation*}
Putting these pieces together with Eq.~\eqref{inequality-descent-adagrad-first} yields that 
\begin{eqnarray}\label{inequality-descent-adagrad-second}
f_\eta(U_{t+1}) - f_\eta(U_t) & \geq & \frac{\gamma}{4\|C\|_\infty}\left(1 - \frac{8\gamma}{\alpha}\left(L_1^2 + 2L_2 + 2\eta^{-1}L_1^2\|C\|_\infty\right)\right)\|\grad f_\eta(U_t)\|_F^2 \nonumber \\
& & \hspace*{-6em} - \frac{\gamma}{\alpha\|C\|_\infty}\left(1 + \gamma(2 L_1^2 + 4L_2 + 4 \eta^{-1}L_1^2\|C\|_\infty)\right)\|G_{t+1} - \grad f_\eta(U_t)\|_F^2. 
\end{eqnarray}
Recall that $G_{t+1} = P_{\Tg_{U_t}\St}(2V_{\pi_{t+1}}U_t)$ and $\grad f_\eta(U_t) = P_{\Tg_{U_t}\St}(2V_{\tilde{\pi}_t^\star}U_t)$. Then we can apply the same argument as in the proof of Lemma~\ref{lemma:obj-progress-grad} and obtain that 
\begin{equation}\label{inequality-descent-adagrad-third}
\|G_{t+1} - \grad f_\eta(U_t)\|_F \ \leq \ \frac{\epsilon\sqrt{\alpha}}{10}. 
\end{equation}
Plugging Eq.~\eqref{inequality-descent-adagrad-third} into Eq.~\eqref{inequality-descent-adagrad-second} with the definition of $\gamma$ yields that 
\begin{equation*}
f_\eta(U_{t+1}) - f_\eta(U_t) \ \geq \ \frac{\gamma\|\grad f_\eta(U_t)\|_F^2}{8\|C\|_\infty} - \frac{\gamma\epsilon^2}{80\|C\|_\infty}.
\end{equation*}
Summing and rearranging the resulting inequality yields that 
\begin{equation*}
\frac{1}{T}\left(\sum_{t=0}^{T-1} \|\grad f_\eta(U_t)\|_F^2\right) \ \leq \ \frac{8\|C\|_\infty(f_\eta(U_T) - f_\eta(U_0))}{\gamma T} + \frac{\epsilon^2}{10}. 
\end{equation*}
This together with the definition of $\Delta_f$ implies the desired result. 
\end{proof}

\subsection{Main results}
Before proceeding to the main results, we present a technical lemma on the Hoffman's bound~\citep{Hoffman-1952-Approximate, Li-1994-Sharp} and the characterization of the Hoffman constant~\citep{Guler-1995-Approximations, Klatte-1995-Error, Wang-2014-Iteration}, which will be also crucial to the subsequent analysis.  
\begin{lemma}\label{lemma:Hoffman}
Consider a polyhedron set $\SCal = \{x \in \br^d \mid Ex=t, x \geq 0\}$. For any point $x \in \br^d$, we have 
\begin{equation*}
\|x - \proj_\SCal(x)\|_1 \leq \theta(E)\left\|\begin{bmatrix} \max\{0, -x\} \\ Ex-t \end{bmatrix}\right\|_1, 
\end{equation*}
where $\theta(E)$ is the Hoffman constant and can be represented by () 
\begin{equation*}
\theta(E) = \sup_{u, v \in \br^d} \left\{\left\|\begin{bmatrix} u \\ v \end{bmatrix}\right\|_\infty \left| \begin{array}{l} \|E^\top v - u\|_\infty = 1, u \geq 0 \\ \textnormal{The corresponding rows of $E$ to $v$’s nonzero} \\ \textnormal{elements are linearly independent.} \end{array}\right. \right\} 
\end{equation*}
\end{lemma}
We then present the iteration complexity of the RGAS algorithm (Algorithm~\ref{alg:grad-sinkhorn}) and the RAGAS algorithm (Algorithm~\ref{alg:adagrad-sinkhorn}) in the following two theorems. 
\begin{theorem}\label{Theorem:RGAS-Total-Iteration}
Letting $\{(U_t, \pi_t)\}_{t \geq 1}$ be the iterates generated by Algorithm~\ref{alg:grad-sinkhorn}, the number of iterations required to reach $\dist(0, \subdiff f(U_t)) \leq \epsilon$ satisfies that 
\begin{equation*}
t \ = \ \bigOtil\left(\frac{k\|C\|_\infty^2}{\epsilon^2}\left(1 + \frac{\|C\|_\infty}{\epsilon}\right)^2\right). 
\end{equation*}
\end{theorem}
\begin{proof}
Let $\tilde{\pi}_t^\star$ be a minimzer of entropy-regularized OT problem and $\pi_t^\star$ be the projection of $\tilde{\pi}_t^\star$ onto the optimal solution set of unregularized OT problem. More specifically, the unregularized OT problem is a LP and the optimal solution set is a polyhedron set ($t^\star$ is an optimal objective value) 
\begin{equation*}
\SCal = \{\pi \in \br^{d \times d} \mid \pi \in \Pi(\mu, \nu), \ \langle U_tU_t^\top, V_\pi\rangle = t^\star\}. 
\end{equation*}
Then we have
\begin{equation*}
\tilde{\pi}_t^\star \in \argmin_{\pi \in \Pi(\mu, \nu)} \ \langle U_tU_t^\top, V_\pi\rangle - \eta H(\pi), \qquad \pi_t^\star = \proj(\tilde{\pi}_t^\star) \in \argmin_{\pi \in \Pi(\mu, \nu)} \ \langle U_tU_t^\top, V_\pi\rangle.   
\end{equation*}
By definition, we have $\nabla f_\eta(U_t) = 2V_{\tilde{\pi}_t^\star}U_t$ and $2V_{\pi_t^\star}U_t \in \partial f(U_t)$. This together with the definition of Riemannian gradient and Riemannian subdifferential yields that 
\begin{eqnarray*}
\grad f_\eta(U_t) & = & P_{\Tg_{U_t}\St}(2V_{\tilde{\pi}_t^\star}U_t), \\
\subdiff f(U_t) & \ni & P_{\Tg_{U_t}\St}(2V_{\pi_t^\star}U_t). 
\end{eqnarray*}
Therefore, we conclude that 
\begin{eqnarray*}
\dist(0, \subdiff f(U_t)) & \leq & \|P_{\Tg_{U_t}\St}(2V_{\pi_t^\star}U_t)\|_F \\ 
& & \hspace*{-6em} \ \leq \ \|P_{\Tg_{U_t}\St}(2V_{\tilde{\pi}_t^\star}U_t)\|_F + \|P_{\Tg_{U_t}\St}(2V_{\pi_t^\star}U_t) - P_{\Tg_{U_t}\St}(2V_{\tilde{\pi}_t^\star}U_t)\|_F \\
& & \hspace*{-6em} \ \leq \ \|\grad f_\eta(U_t)\|_F + 2\|(V_{\pi_t^\star} - V_{\tilde{\pi}_t^\star})U_t\|_F. 
\end{eqnarray*}
Note that scaling the objective function by $\|C\|_\infty$ will not change the optimal solution set. Since $U_t \in \St(d, k)$, each entry of the coefficient in the normalized objective function is less than 1. By Lemma~\ref{lemma:Hoffman}, we obtain that there exists a constant $\bar{\theta}$ independent of $\|C\|_\infty$ such that  
\begin{equation*}
\|\tilde{\pi}_t^\star - \pi_t^\star\|_1 \ \leq \ \bar{\theta}\left\|\left\langle U_tU_t^\top, \frac{V_{\tilde{\pi}_t^\star} - V_{\pi_t^\star}}{\|C\|_\infty}\right\rangle\right\|_1. 
\end{equation*}
By the definition of $\tilde{\pi}_t^\star$, we have $\langle U_tU_t^\top, V_{\tilde{\pi}_t^\star}\rangle - \eta H(\tilde{\pi}_t^\star) \leq \langle U_tU_t^\top, V_{\pi_t^\star}\rangle - \eta H(\pi_t^\star)$. Since $0 \leq H(\pi) \leq 2\log(n)$ and $\eta = \frac{\epsilon\min\{1, 1/\bar{\theta}\}}{40\log(n)}$, we have
\begin{equation*}
\tilde{\pi}_t^\star \ \in \ \Pi(\mu, \nu), \qquad 0 \ \leq \ \langle U_tU_t^\top, V_{\tilde{\pi}_t^\star} - V_{\pi_t^\star}\rangle \ \leq \ \epsilon/(20\bar{\theta}). 
\end{equation*}
Putting these pieces together yields that 
\begin{equation*}
\|\tilde{\pi}_t^\star - \pi_t^\star\|_1 \ \leq \ \frac{\epsilon}{20\|C\|_\infty\bar{\theta}}. 
\end{equation*}
By the definition of $U_t$ and $V_\pi$, we have
\begin{equation*}
\|(V_{\pi_t^\star} - V_{\tilde{\pi}_t^\star})U_t\|_F \ = \ \|V_{\pi_t^\star} - V_{\tilde{\pi}_t^\star}\|_F \ \leq \ \bar{\theta}\|C\|_\infty\|\tilde{\pi}_t^\star - \pi_t^\star\|_1 \ \leq \ \frac{\epsilon}{20}. 
\end{equation*}
Putting these pieces together yields 
\begin{equation*}
\dist(0, \subdiff f(U_t)) \ \leq \ \|\grad f_\eta(U_t)\|_F + \frac{\epsilon}{10}. 
\end{equation*}
Combining this inequality with Lemma~\ref{lemma:obj-progress-grad} and the Cauchy-Schwarz inequality, we have
\begin{eqnarray*}
\frac{1}{T}\left(\sum_{t=0}^{T-1} [\dist(0, \subdiff f(U_t))]^2\right) & \leq & \frac{2}{T}\left(\sum_{t=0}^{T-1} \|\grad f_\eta(U_t)\|_F^2\right) + \frac{\epsilon^2}{50} \ \leq \ \frac{8\Delta_f}{\gamma T} + \frac{2\epsilon^2}{5} + \frac{\epsilon^2}{50} \\
& \leq & \frac{8\Delta_f}{\gamma T} + \frac{\epsilon^2}{2}.  
\end{eqnarray*}
Given that $\dist(0, \subdiff f(U_t)) > \epsilon$ for all $t = 0, 1, \ldots, T-1$ and 
\begin{equation*}
\frac{1}{\gamma} \ = \ (8 L_1^2 + 16L_2)\|C\|_\infty + \frac{16 L_1^2\|C\|_\infty^2}{\eta} \ = \ (8 L_1^2 + 16L_2)\|C\|_\infty + \frac{640 L_1^2\max\{1, \bar{\theta}\}\|C\|_\infty^2\log(n)}{\epsilon}. 
\end{equation*} 
we conclude that the upper bound $T$ must satisfy
\begin{equation*}
\epsilon^2 \ \leq \ \frac{16\Delta_f}{T}\left((8 L_1^2 + 16L_2)\|C\|_\infty + \frac{640 L_1^2\max\{1, \bar{\theta}\}\|C\|_\infty^2\log(n)}{\epsilon}\right). 
\end{equation*}
Using Lemma~\ref{lemma:key-inequality}, we have
\begin{eqnarray*}
\Delta_f & \leq & \left( \|C\|_\infty + \frac{2 \|C\|_\infty^2}{\eta}\right)\left(\max_{U \in \St(d, k)} \|U - U_0\|_F^2\right) + 2\|C\|_\infty\left(\max_{U \in \St(d, k)} \|U - U_0\|_F\right) \\ 
& = & k\left(6\|C\|_\infty + \frac{4 \|C\|_\infty^2}{\eta}\right) \ = \ k\left(6\|C\|_\infty + \frac{160\max\{1, \bar{\theta}\}\|C\|_\infty^2\log(n)}{\epsilon}\right).
\end{eqnarray*}
Putting these pieces together implies the desired result. 
\end{proof}
\begin{theorem}\label{Theorem:RAGAS-Total-Iteration}
Letting $\{(U_t, \pi_t)\}_{t \geq 1}$ be the iterates generated by Algorithm~\ref{alg:adagrad-sinkhorn}, the number of iterations required to reach $\dist(0, \subdiff f(U_t)) \leq \epsilon$ satisfies 
\begin{equation*}
t \ = \ \bigOtil\left(\frac{k\|C\|_\infty^2}{\epsilon^2}\left(1 + \frac{\|C\|_\infty}{\epsilon}\right)^2\right). 
\end{equation*}
\end{theorem}
\begin{proof}
Using the same argument as in the proof of Theorem~\ref{Theorem:RGAS-Total-Iteration}, we have
\begin{equation*}
\dist(0, \subdiff f(U_t)) \ \leq \ \|\grad f_\eta(U_t)\|_F + \frac{\epsilon}{10}. 
\end{equation*}
Combining this inequality with Lemma~\ref{lemma:obj-progress-adagrad} and the Cauchy-Schwarz inequality, we have
\begin{eqnarray*}
\frac{1}{T}\left(\sum_{t=0}^{T-1} [\dist(0, \subdiff f(U_t))]^2\right) & \leq & \frac{2}{T}\left(\sum_{t=0}^{T-1} \|\grad f_\eta(U_t)\|_F^2\right) + \frac{\epsilon^2}{50} \\
& & \hspace*{-10em} \leq \ \frac{16\|C\|_\infty\Delta_f}{\gamma T} + \frac{\epsilon^2}{5} + \frac{\epsilon^2}{50} \ \leq \ \frac{16\|C\|_\infty\Delta_f}{\gamma T} + \frac{\epsilon^2}{2}.  
\end{eqnarray*}
Given that $\dist(0, \subdiff f(U_t)) > \epsilon$ for all $t = 0, 1, \ldots, T-1$ and 
\begin{equation*}
\frac{1}{\gamma} \ = \ 16L_1^2 + 32L_2 + \frac{1280L_1^2\max\{1, \bar{\theta}\}\|C\|_\infty\log(n)}{\epsilon}, 
\end{equation*} 
we conclude that the upper bound $T$ must satisfies 
\begin{equation*}
\epsilon^2 \ \leq \ \frac{64\|C\|_\infty\Delta_f}{T}\left(16L_1^2 + 32L_2 + \frac{1280L_1^2\max\{1, \bar{\theta}\}\|C\|_\infty\log(n)}{\epsilon}\right). 
\end{equation*}
Using Lemma~\ref{lemma:key-inequality}, we have
\begin{eqnarray*}
\Delta_f & \leq & \left( \|C\|_\infty + \frac{2 \|C\|_\infty^2}{\eta}\right)\left(\max_{U \in \St(d, k)} \|U - U_0\|_F^2\right) \ = \ k\left(2 \|C\|_\infty + \frac{4 \|C\|_\infty^2}{\eta}\right) \\
& = & k\left(2 \|C\|_\infty + \frac{160\max\{1, \bar{\theta}\}\|C\|_\infty^2\log(n)}{\epsilon}\right). 
\end{eqnarray*}
Putting these pieces together implies the desired result. 
\end{proof}
From Theorem~\ref{Theorem:RGAS-Total-Iteration} and~\ref{Theorem:RAGAS-Total-Iteration}, Algorithm~\ref{alg:grad-sinkhorn} and~\ref{alg:adagrad-sinkhorn} achieve the same iteration complexity. Furthermore, the number of arithmetic operations at each loop of Algorithm~\ref{alg:grad-sinkhorn} and~\ref{alg:adagrad-sinkhorn} are also the same. Thus, the complexity bound of Algorithm~\ref{alg:adagrad-sinkhorn} is the same as that of Algorithm~\ref{alg:grad-sinkhorn}. 

\begin{theorem}\label{Theorem:RGAS-RAGAS-Total-Complexity}
Either the RGAS algorithm or the RAGAS algorithm returns an $\epsilon$-approximate pair of optimal subspace projection and optimal transportation plan of the computation of the PRW distance in Eq.~\eqref{prob:main} (cf. Definition~\ref{def:optimal-pair}) in
\begin{equation*}
\bigOtil\left(\left(\frac{n^2d\|C\|_\infty^2}{\epsilon^2} + \frac{n^2\|C\|_\infty^6}{\epsilon^6} + \frac{n^2\|C\|_\infty^{10}}{\epsilon^{10}}\right)\left(1 + \frac{\|C\|_\infty}{\epsilon}\right)^2\right)
\end{equation*}
arithmetic operations. 
\end{theorem}
\begin{proof}
First, Theorem~\ref{Theorem:RGAS-Total-Iteration} and~\ref{Theorem:RAGAS-Total-Iteration} imply that both algorithms achieve the same the iteration complexity as follows, 
\begin{equation}\label{RGAS-RAGAS-iteration}
t \ = \ \bigOtil\left(\frac{k\|C\|_\infty^2}{\epsilon^2}\left(1 + \frac{\|C\|_\infty}{\epsilon}\right)^2\right). 
\end{equation}
This implies that $U_t$ is an $\epsilon$-approximate optimal subspace projection of problem~\eqref{prob:Stiefel-nonsmooth}. By the definition of $\widehat{\epsilon}$ and using the stopping criterion of the subroutine $\textsc{regOT}(\{(x_i, r_i)\}_{i \in [n]}, \{(y_j, c_j)\}_{j \in [n]}, U_t, \eta, \widehat{\epsilon})$, we have $\pi_{t+1} \in \Pi(\mu, \nu)$ and 
\begin{equation*}
0 \ \leq \ \langle U_tU_t^\top, V_{\pi_{t+1}} - V_{\tilde{\pi}_t^\star}\rangle \ \leq \ \|C\|_\infty\|\pi_{t+1} - \tilde{\pi}_t^\star\|_1 \ \leq \ \|C\|_\infty\widehat{\epsilon} \ \leq \ \epsilon/2. 
\end{equation*}
where $\tilde{\pi}_t^\star$ is an unique minimzer of entropy-regularized OT problem. Furthermore, by the definition of $\tilde{\pi}_t^\star$, we have $\langle U_tU_t^\top, V_{\tilde{\pi}_t^\star}\rangle - \eta H(\tilde{\pi}_t^\star) \leq \langle U_tU_t^\top, V_{\pi_t^\star}\rangle - \eta H(\pi_t^\star)$. Since $0 \leq H(\pi) \leq 2\log(n)$ and $\eta = \frac{\epsilon\min\{1, 1/\bar{\theta}\}}{40\log(n)}$, we have
\begin{equation*}
\tilde{\pi}_t^\star \ \in \ \Pi(\mu, \nu), \qquad 0 \ \leq \ \langle U_tU_t^\top, V_{\tilde{\pi}_t^\star} - V_{\pi_t^\star}\rangle \ \leq \ \epsilon/2. 
\end{equation*}
Putting these pieces together yields that $\pi_{t+1}$ is an $\epsilon$-approximate optimal transportation plan for the subspace projection $U_t$. Therefore, we conclude that $(U_t, \pi_{t+1}) \in \St(d, k) \times \Pi(\mu, \nu)$ is an \emph{$\epsilon$-approximate pair of optimal subspace projection and optimal transportation plan} of problem~\eqref{prob:main}. 

The remaining step is to analyze the complexity bound. Indeed, we first claim that the number of arithmetic operations required by the Sinkhorn iteration at each loop is upper bounded by 
\begin{equation}\label{RGAS-complexity-Sinkhorn}
\bigOtil\left(\frac{n^2\|C\|_\infty^4}{\epsilon^4} + \frac{n^2\|C\|_\infty^8}{\epsilon^8}\right).
\end{equation}
Furthermore, while \textbf{Step 5} and \textbf{Step 6} in Algorithm~\ref{alg:grad-sinkhorn} can be implemented in $\bigO(dk^2 + k^3)$ arithmetic operations, we still need to construct $V_{\pi_{t+1}} U_t$. A naive approach suggests to first construct $V_{\pi_{t+1}}$ using $\bigO(n^2 dk)$ arithmetic operations and then perform the matrix multiplication using $\bigO(d^2k)$ arithmetic operations. This is computationally prohibitive since $d$ can be very large in practice. In contrast, we observe that 
\begin{equation*}
V_{\pi_{t+1}} U_t \ = \ \sum_{i=1}^n \sum_{j=1}^n (\pi_{t+1})_{i, j} (x_i - y_j)(x_i - y_j)^\top U_t. 
\end{equation*}
Since $x_i - y_j \in \br^d$, it will take $\bigO(dk)$ arithmetic operations for computing $(x_i - y_j)(x_i - y_j)^\top U_t$ for all $(i, j) \in [n] \times n$. This implies that the total number of arithmetic operations is $\bigO(n^2dk)$. Therefore, the number of arithmetic operations at each loop is
\begin{equation}\label{RGAS-arithmetic-operation}
\bigOtil\left(n^2 dk + dk^2 + k^3 + \frac{n^2\|C\|_\infty^4}{\epsilon^4} + \frac{n^2\|C\|_\infty^8}{\epsilon^8}\right). 
\end{equation}
Putting Eq.~\eqref{RGAS-RAGAS-iteration} and Eq.~\eqref{RGAS-arithmetic-operation} together with $k = \bigOtil(1)$ yields the desired result.

\paragraph{Proof of claim~\eqref{RGAS-complexity-Sinkhorn}.} The proof is based on the combination of several existing results proved by~\citet{Altschuler-2017-Near} and~\citet{Dvurechensky-2018-Computational}. For the sake of completeness, we provide the details. More specifically, we consider solving the entropic regularized OT problem as follows, 
\begin{equation*}
\min_{\pi \in \br_+^{n \times n}} \ \langle C, \pi\rangle - \eta H(\pi), \quad \st \ r(\pi) = r, \ c(\pi) = c.
\end{equation*}
We leverage the Sinkhorn iteration which aims at minimizing the following function 
\begin{equation*}
f(u, v) = \one_n^\top B(u, v)\one_n - \langle u, r\rangle - \langle v, c\rangle, \quad \textnormal{where } B(u, v) := \diag(u)e^{-\frac{C}{\eta}}\diag(v). 
\end{equation*}
From the update scheme of Sinkhorn iteration, it is clear that $\one_n^\top B(u_j, v_j)\one_n = 1$ for each iteration $j$. By a straightforward calculation, we have 
\begin{eqnarray*}
& & \langle C, B(u_j, v_j)\rangle - \eta H(B(u_j, v_j)) - \left(\langle C, B(u^\star, v^\star)\rangle - \eta H(B(u^\star, v^\star))\right) \\
& \leq & \eta(f(u_j, v_j) - f(u^\star, v^\star)) + \eta R(\|r(B(u_j, v_j))-r\|_1 + \|c(B(u_j, v_j))-c\|_1)
\end{eqnarray*}
where $(u^\star, v^\star)$ is a maximizer of $f(u, v)$ over $\br^n \times \br^n$ and $R>0$ is defined in~\citet[Lemma~1]{Dvurechensky-2018-Computational}. Since the entropic regularization function is strongly convex with respect to $\ell_1$-norm over the probability simplex and $B(u_j, v_j)$ can be vectorized as a probability vector, we have
\begin{equation*}
\|B(u_j, v_j) - B(u^\star, v^\star)\|_1^2 \ \leq \ 2(f(u_j, v_j) - f(u^\star, v^\star)) + 2R(\|r(B(u_j, v_j))-r\|_1 + \|c(B(u_j, v_j))-c\|_1). 
\end{equation*}
On one hand, by the definition of $(u^\star, v^\star)$ and $B(\cdot, \cdot)$, it is clear that $B(u^\star, v^\star)$ is an unique optimal solution of the entropic regularized OT problem and we further denote it as $\tilde{\pi}^\star$. On the other hand, the final output $\pi \in \Pi(\mu, \nu)$ is achieved by rounding $B(u_j, v_j)$ to $\Pi(\mu, \nu)$ for some $j$ using~\citet[Algorithm~2]{Altschuler-2017-Near} and~\citet[Lemma~7]{Altschuler-2017-Near} guarantees that 
\begin{equation*}
\|\tilde{\pi} - B(u_j, v_j)\|_1 \leq 2(\|r(B(u_j, v_j))-r\|_1 + \|c(B(u_j, v_j))-c\|_1). 
\end{equation*}
Again, from the update scheme of Sinkhorn iteration and By Pinsker's inequality, we have
\begin{equation*}
\sqrt{2\left(f(u_j, v_j) - f(u^\star, v^\star)\right)} \ \geq \ \|r(B(u_j, v_j))-r\|_1 + \|c(B(u_j, v_j))-c\|_1.  
\end{equation*}
Putting these pieces together yields that 
\begin{equation*}
\|\tilde{\pi} - \tilde{\pi}^\star\|_1 \ \leq \ c_1\left(f(u_j, v_j) - f(u^\star, v^\star)\right)^{1/2} +c_2\sqrt{R}\left(f(u_j, v_j) - f(u^\star, v^\star)\right)^{1/4}
\end{equation*}
where $c_1, c_2 > 0$ are constants. Then, by using Eq.(12) in~\citet[Theorem~1]{Dvurechensky-2018-Computational}, we have $f(u_j, v_j) - f(u^\star, v^\star) \leq \frac{2R^2}{j}$. This together with the definition of $R$ yields that the number of iterations required by the Sinkhorn iteration is 
\begin{equation*}
\bigOtil\left(\frac{\|C\|_\infty^4}{\epsilon^4} + \frac{\|C\|_\infty^8}{\epsilon^8}\right).
\end{equation*}
This completes the proof. 
\end{proof}
\begin{remark}
Theorem~\ref{Theorem:RGAS-RAGAS-Total-Complexity} is surprising in that it provides a finite-time guarantee for finding an $\epsilon$-stationary point of a nonsmooth function $f$ over a nonconvex constraint set. This is impossible for general nonconvex nonsmooth optimization even in the Euclidean setting~\citep{Zhang-2020-Complexity, Shamir-2020-Can}. Our results demonstrate that the max-min optimization model in Eq.~\eqref{prob:main} has a special structure that makes fast computation possible. 
\end{remark} 
\begin{remark}
Note that our algorithms only return an approximate stationary point for the nonconvex max-min optimization model in Eq.~\eqref{prob:main}, which needs to be evaluated in practice. It is also interesting to compare such stationary point to the global optimal solution of computing the SRW distance. This is very challenging in general due to multiple stationary points of non-convex max-min optimization model in Eq.~\eqref{prob:main} but possible if the data has certain structure. We leave it to the future work. 
\end{remark}

\begin{figure}[!t]
\centering\hspace*{-1em}
\includegraphics[clip, trim=0 0 0 0, width=.96\textwidth]{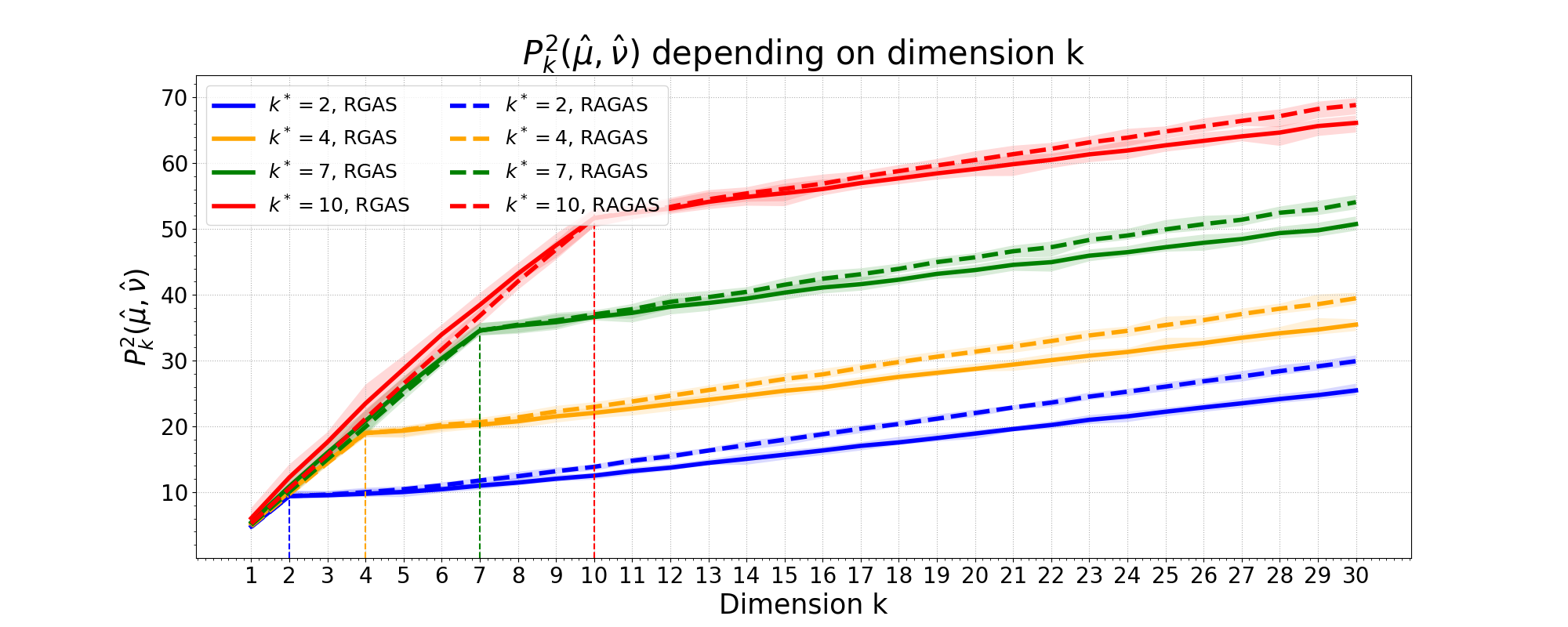}
\caption{Computation of $\PCal_k^2(\hat{\mu},\hat{\nu})$ depending on the dimension $k \in [d]$ and $k^* \in \{2, 4, 7, 10\}$, where $\hat{\mu}$ and $\hat{\nu}$ stand for the empirical measures of $\mu$ and $\nu$ with 100 points. The solid and dash curves are the computation of $\PCal_k^2(\hat{\mu},\hat{\nu})$ with the RGAS and RAGAS algorithms, respectively. Each curve is the mean over 100 samples with shaded area covering the min and max values.}\label{fig:exp1_dim_k}
\end{figure}

\section{Experiments}\label{sec:experiment}
We conduct numerical experiments to evaluate the computation of the PRW distance by the RGAS and RAGAS algorithms. The baseline approaches include the computation of SRW distance with the Frank-Wolfe algorithm\footnote{Available in https://github.com/francoispierrepaty/SubspaceRobustWasserstein.}~\citep{Paty-2019-Subspace} and the computation of Wasserstein distance with the POT software package\footnote{Available in https://github.com/PythonOT/POT}~\citep{Flamary-2017-Pot}. For the RGAS and RAGAS algorithms, we set $\gamma = 0.01$ unless stated otherwise, $\beta = 0.8$ and $\alpha=10^{-6}$. For the experiments on the MNIST digits, we run the feature extractor pretrained in PyTorch 1.5. All the experiments are implemented in Python 3.7 with Numpy 1.18 on a ThinkPad X1 with an Intel Core i7-10710U (6 cores and 12 threads) and 16GB memory, equipped with Ubuntu 20.04.
\begin{figure}[!t]
\centering
\includegraphics[width=0.45\textwidth]{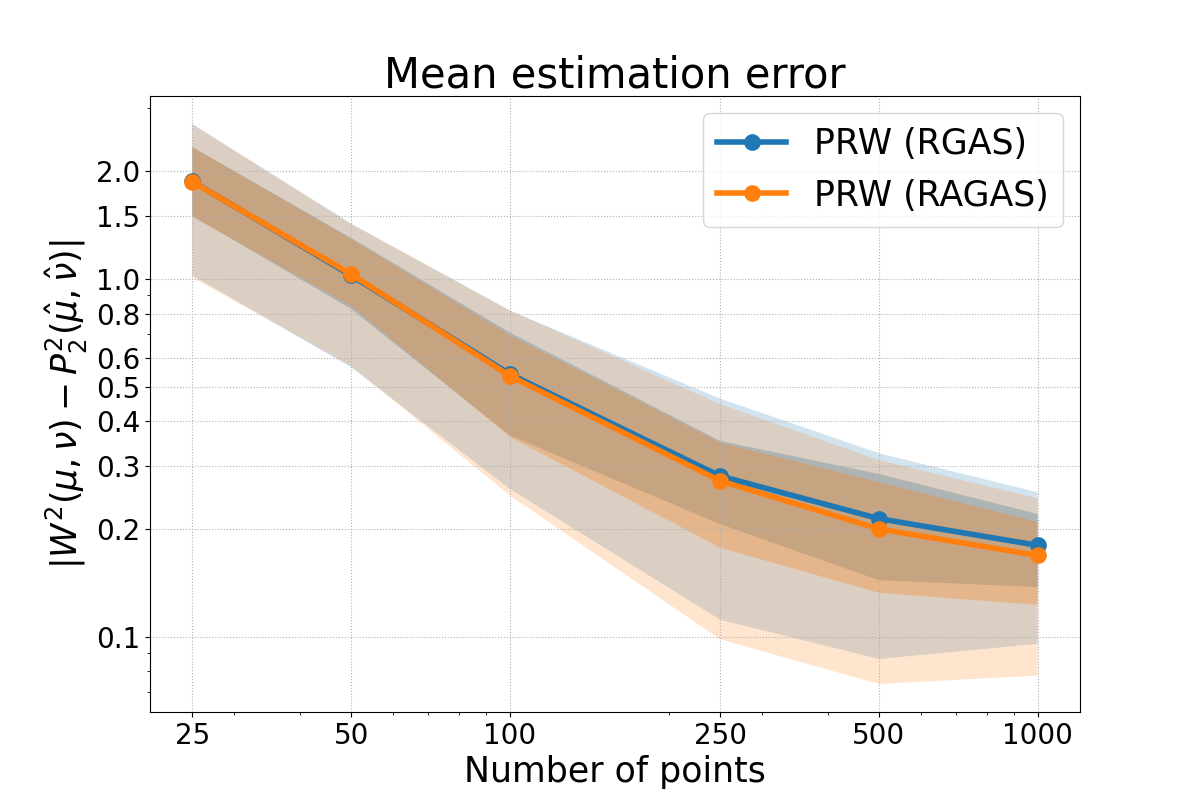}
\includegraphics[width=0.45\textwidth]{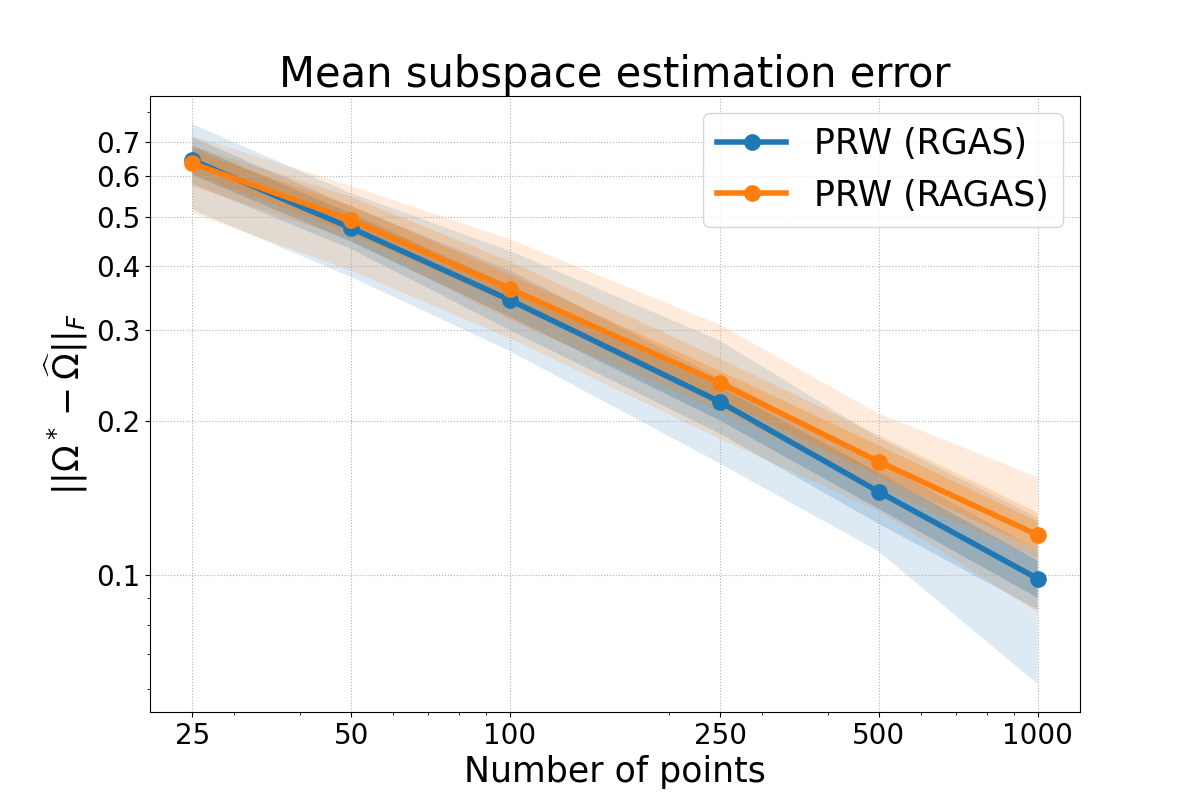}
\caption{Mean estimation error (left) and mean subspace estimation error (right), with varying number of points $n$. The shaded areas represent the 10\%-90\% and 25\%-75\% quantiles over 100 samples.}\label{fig:exp_cube_1}
\end{figure}
\begin{figure}[!t]
\centering
\includegraphics[clip, trim=10 10 10 10, width=0.32\textwidth]{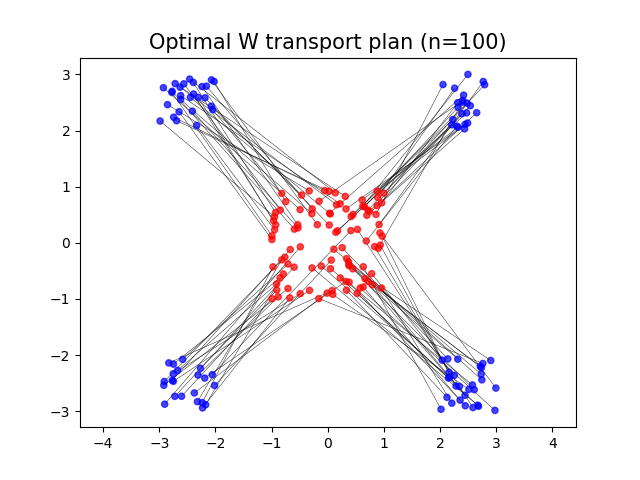}
\includegraphics[clip, trim=10 10 10 10, width=0.32\textwidth]{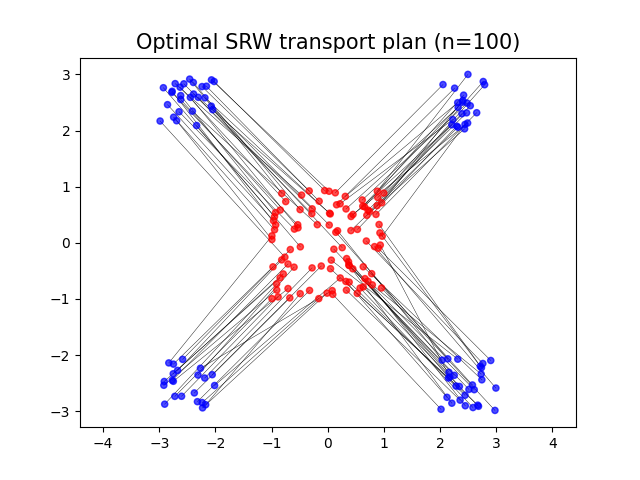}
\includegraphics[clip, trim=10 10 10 10, width=0.32\textwidth]{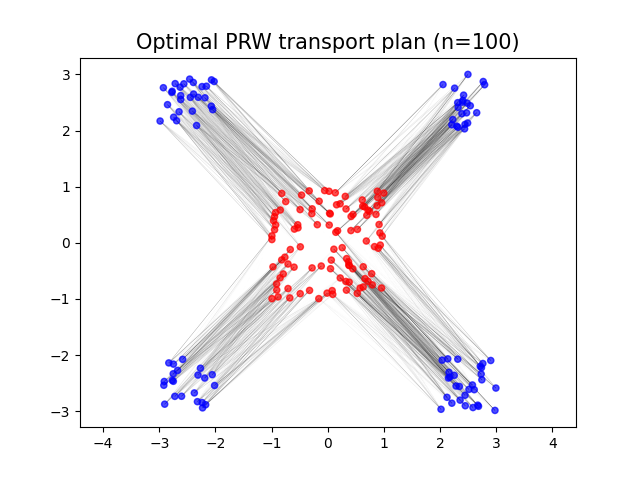}
\includegraphics[clip, trim=10 10 10 10, width=0.32\textwidth]{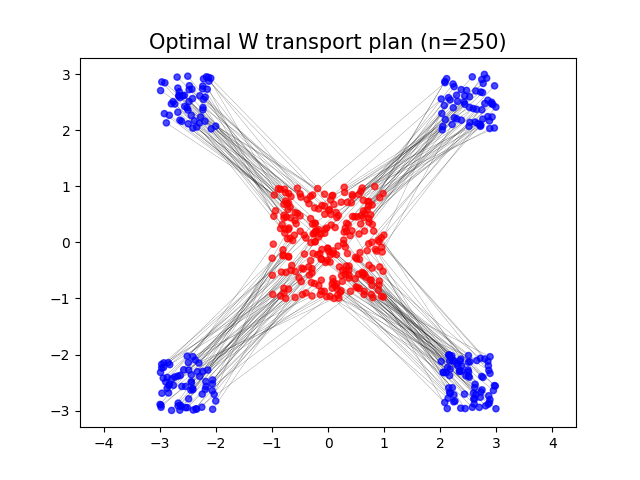}
\includegraphics[clip, trim=10 10 10 10, width=0.32\textwidth]{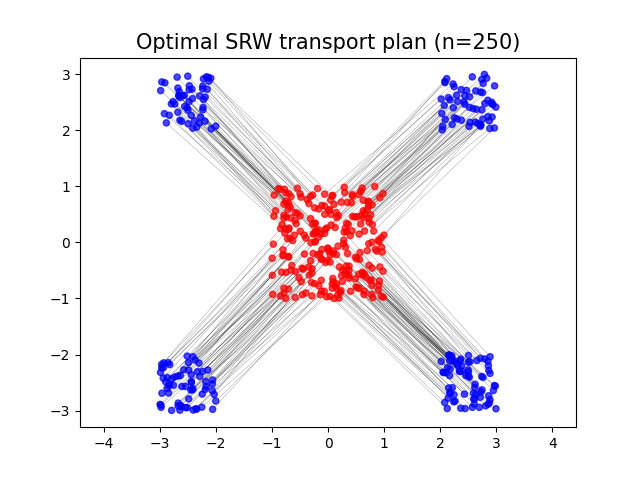}
\includegraphics[clip, trim=10 10 10 10, width=0.32\textwidth]{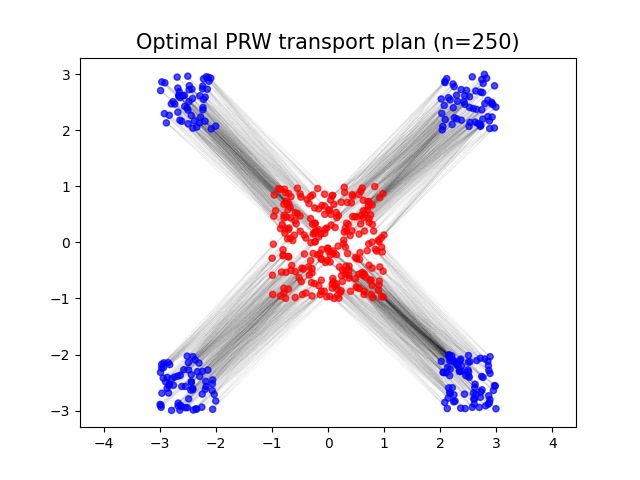}
\caption{Fragmented hypercube with $(n, d) = (100, 30)$ (above) and $(n, d) = (250, 30)$ (bottom). Optimal mappings in the Wasserstein space (left), in the SRW space (middle) and the PRW space (right). Geodesics in the PRW space are robust to statistical noise.}
\label{fig:exp_cube_2}
\end{figure}
\paragraph{Fragmented hypercube.} We conduct our first experiment on the fragmented hypercube which is also used to evaluate the SRW distance~\citep{Paty-2019-Subspace} and FactoredOT~\citep{Forrow-2019-Statistical}. In particular, we consider $\mu = \UCal([-1,1]^d)$ which is an uniform distribution over an hypercube and $\nu = T_{\#} \mu$ which is the push-forward of $\mu$ under the map $T(x) = x + 2\sign(x)\odot(\sum_{k=1}^{k^*} e_k)$. Note that $\sign(\cdot)$ is taken element-wise, $k^* \in [d]$ and $(e_1, \ldots, e_d)$ is the canonical basis of $\br^d$. By the definition, $T$ divides $[-1,1]^d$ into four different hyper-rectangles, as well as serves as a subgradient of convex function. This together with Brenier's theorem (cf.\ \citet[Theorem~2.12]{Villani-2003-Topics}) implies that $T$ is an optimal transport map between $\mu$ and $\nu = T_{\#}\mu$ with $\WCal_2^2(\mu, \nu) = 4k^*$. Notice that the displacement vector $T(x) - x$ is optimal for any $x \in \br^d$ and always belongs to the $k^*$-dimensional subspace spanned by $\{e_j\}_{j \in [k^*]}$. Putting these pieces together yields that $\PCal_k^2(\mu, \nu) = 4k^*$ for any $k \geq k^*$. 

Figure~\ref{fig:exp1_dim_k} presents the behavior of $\PCal_k^2(\widehat{\mu}, \widehat{\nu})$ as a function of $k^* \in \{2, 4, 7, 10\}$, where $\widehat{\mu}$ and $\widehat{\nu}$ are empirical distributions corresponding to $\mu$ and $\nu$, respectively. The sequence is concave and increases slowly after $k = k^*$, which makes sense since the last $d - k^*$ dimensions only represent noise. The rigorious argument for the SRW distance is presented in~\citet[Proposition~3]{Paty-2019-Subspace} but hard to be extended here since the PRW distance can not be characterized as a sum of eigenvalues.

Figure~\ref{fig:exp_cube_1} presents mean estimation error and mean subspace estimation error with varying number of points $n \in \{25, 50, 100, 250, 500, 1000\}$. In particular, $\widehat{U}$ is an approximate optimal subspace projection achieved by computing $\PCal_k^2(\widehat{\mu}, \widehat{\nu})$ with our algorithms and $\Omega^*$ is the optimal projection matrix onto the $k^*$-dimensional subspace spanned by $\{e_j\}_{j \in [k^*]}$. We set $k^*=2$ here and $\widehat{\mu}$ and $\hat{\nu}$ are constructed from $\mu$ and $\nu$ respectively with $n$ points each. The quality of solutions obtained by the RGAS and RAGAS algorithms are roughly the same.

Figure~\ref{fig:exp_cube_2} presents the optimal transport plan in the Wasserstein space (left), the optimal transport plan in the SRW space (middle), and the optimal transport plan in the PRW space (right) between $\widehat{\mu}$ and $\widehat{\nu}$. We consider two cases: $n = 100$ and $n = 250$, in our experiment and observe that our results are consistent with~\citet[Figure~5]{Paty-2019-Subspace}, showing that both PRW and SRW distances share important properties with the Wasserstein distance.  

\paragraph{Robustness of $\PCal_k$ to noise.} We conduct our second experiment on the Gaussian distribution\footnote{~\citet{Paty-2019-Subspace} conducted this experiment with their projected supergradient ascent algorithm (cf.\ ~\citet[Algorithm 1]{Paty-2019-Subspace}) with the \textsc{emd} solver from the POT software package. For a fair comparison, we use Riemannian supergradient ascent algorithm (cf.\ Algorithm~\ref{alg:supergrad-simplex}) with the \textsc{emd} solver here; see Appendix for the details.}. In particular, we consider $\mu = \NCal(0, \Sigma_1)$ and $\nu = \NCal(0, \Sigma_2)$ where $\Sigma_1, \Sigma_2 \in \br^{d \times d}$ are positive semidefinite  matrices of rank $k^*$. This implies that either of the support of $\mu$ and $\nu$ is the $k^*$-dimensional subspace of $\br^d$. Even though the supports of $\mu$ and $\nu$ can be different, their union is included in a $2k^*$-dimensional subspace. Putting these pieces together yields that $\PCal_k^2(\mu, \nu) = \WCal_2^2(\mu, \nu)$ for any $k \geq 2k^*$. In our experiment, we set $d = 20$ and sample 100 independent couples of covariance matrices $(\Sigma_1, \Sigma_2)$, where each has independently a Wishart distribution with $k^* = 5$ degrees of freedom. Then we construct the empirical measures $\widehat{\mu}$ and $\widehat{\nu}$ by drawing $n = 100$ points from $\NCal(0, \Sigma_1)$ and $\NCal(0, \Sigma_2)$. 

Figure~\ref{fig:exp_noise_1} presents the mean value of $\SCal_k^2(\widehat{\mu},\widehat{\nu})/\WCal_2^2(\widehat{\mu}, \widehat{\nu})$ (left) and $\PCal_k^2(\widehat{\mu},\widehat{\nu})/\WCal_2^2(\widehat{\mu}, \widehat{\nu})$ (right) over 100 samples with varying $k$. We plot the curves for both noise-free and noisy data, where white noise ($\NCal(0, I_d)$) was added to each data point. With moderate noise, the data is approximately on two $5$-dimensional subspaces and both the SRW and PRW distances do not vary too much. Our results are consistent with the SRW distance presented in~\citet[Figure~6]{Paty-2019-Subspace}, showing that the PRW distance is also robust to random perturbation of the data. 
\begin{figure}[!t]
\centering
\includegraphics[clip, trim=0 0 20 0, width=0.45\textwidth]{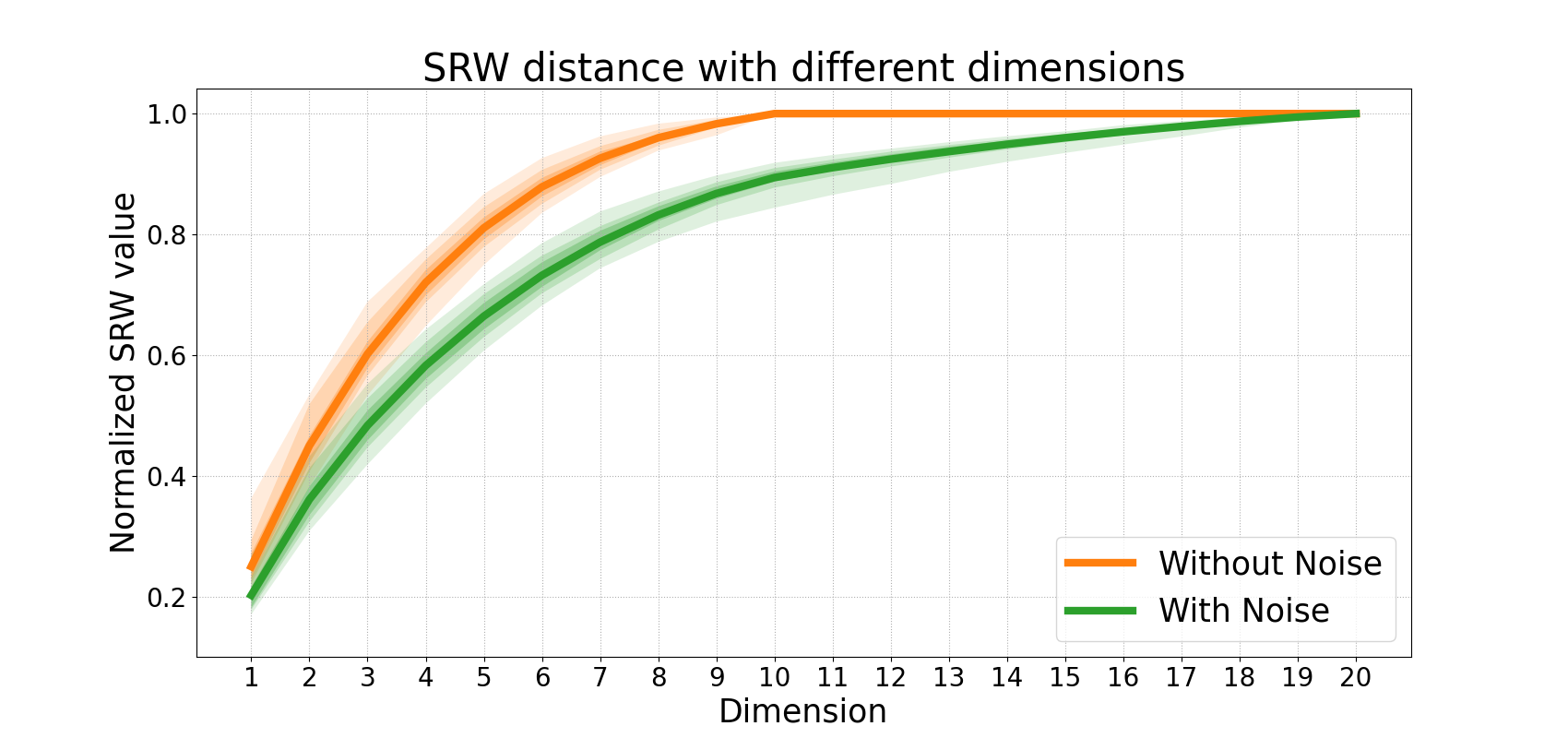}
\includegraphics[clip, trim=20 0 0 0, width=0.45\textwidth]{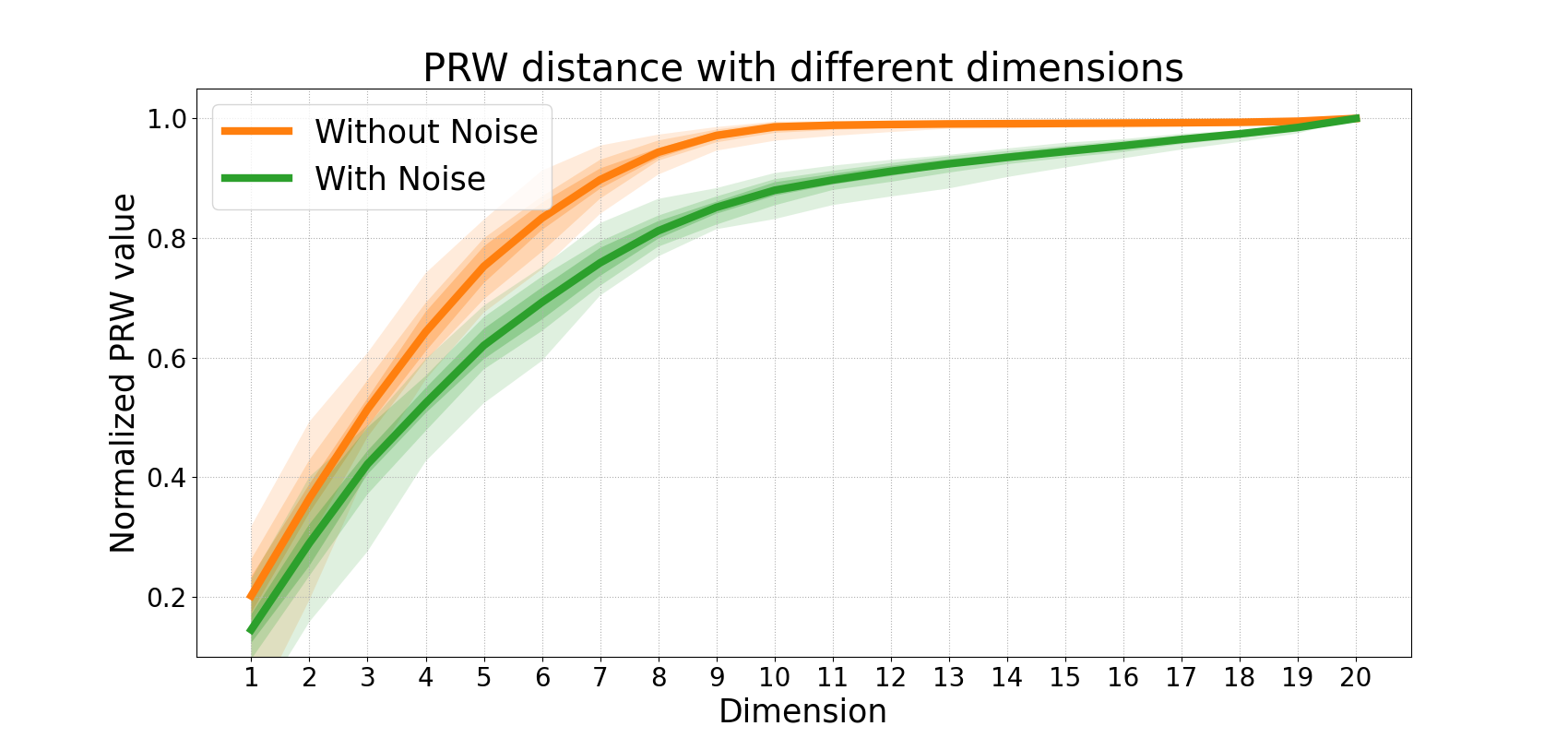}
\caption{Mean normalized SRW distance (left) and mean normalized PRW distance (right) as a function of dimension. The shaded area shows the 10\%-90\% and 25\%-75\% quantiles over the 100 samples.}\label{fig:exp_noise_1}
\end{figure}
\begin{figure}[!t]
\centering
\includegraphics[width=0.45\textwidth]{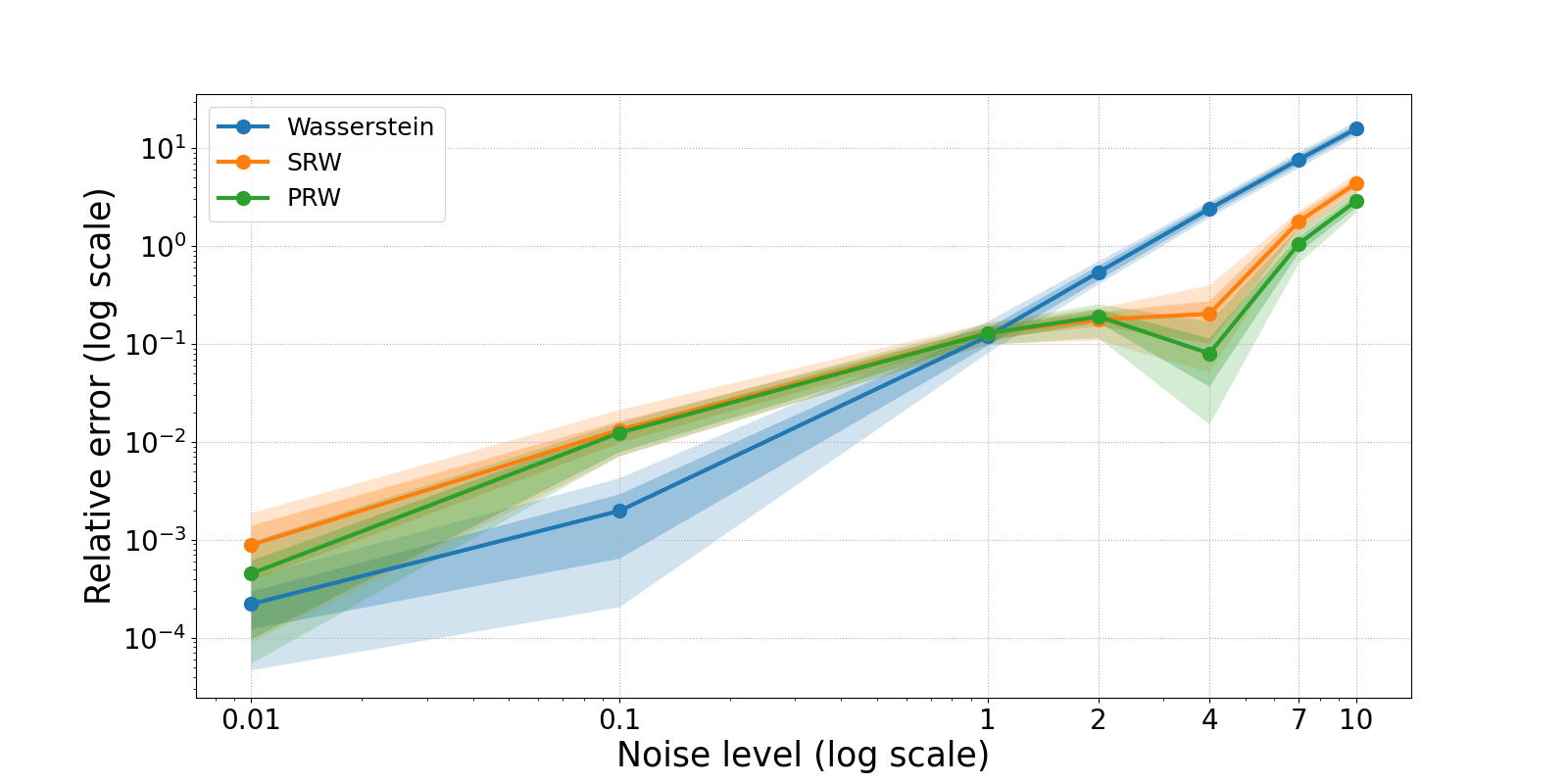}
\includegraphics[width=0.45\textwidth]{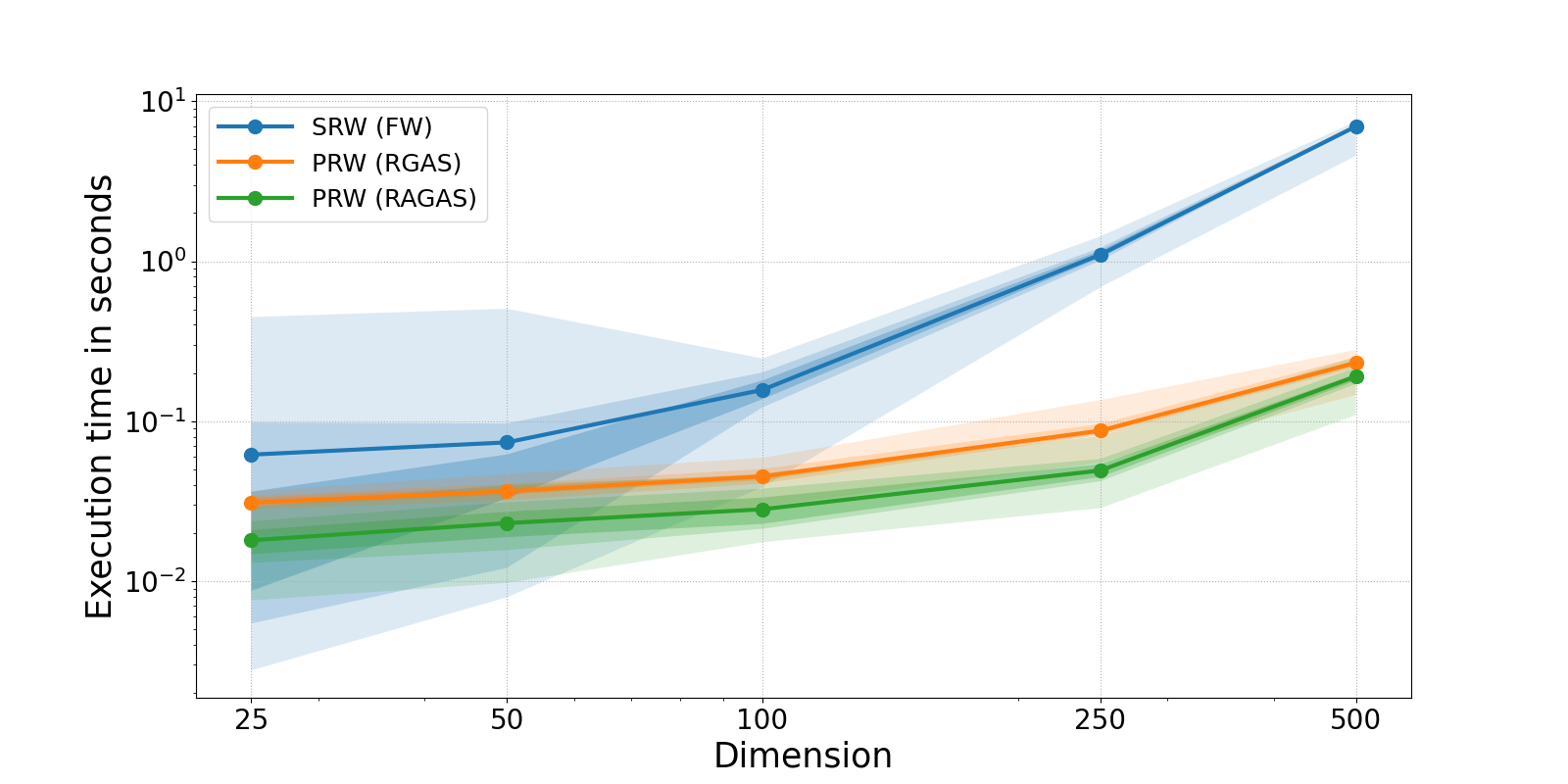}
\caption{(Left) Comparison of mean relative errors over 100 samples, depending on the noise level. The shaded areas show the min-max values and the 10\%-90\% quantiles; (Right) Comparisons of mean computation times on CPU. The shaded areas show the minimum and maximum values over 50 runs.}\label{fig:exp_noise_2}
\end{figure}

Figure~\ref{fig:exp_noise_2} (left) presents the comparison of mean relative errors over 100 samples as the noise level varies. In particular, we construct the empirical measures $\widehat{\mu}_\sigma$ and $\hat{\nu}_\sigma$ by gradually adding Gaussian noise $\sigma\NCal(0, I_d)$ to the points. The relative errors of the Wasserstein, SRW and PRW distances are defined the same as in~\citet[Section~6.3]{Paty-2019-Subspace}. For small noise level, the imprecision in the computation of the SRW distance adds to the error caused by the added noise, while the computation of the PRW distance with our algorithms is less sensitive to such noise. When the noise has the moderate to high variance, the PRW distance is the most robust to noise, followed by the SRW distance, both of which outperform the Wasserstein distance.

\begin{wrapfigure}{r}{0.5\textwidth}
\vspace*{-.5em}\includegraphics[width=0.45\textwidth]{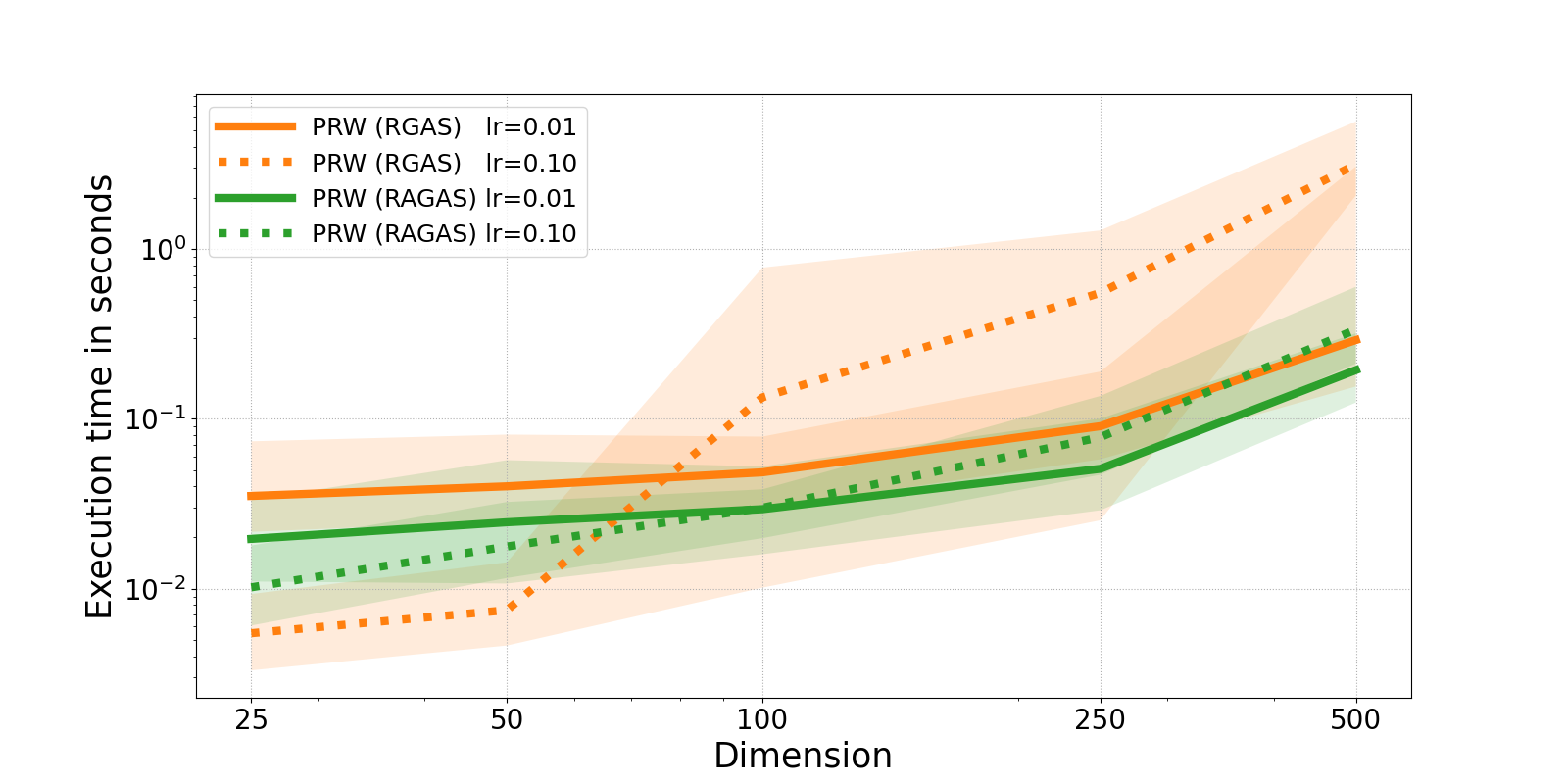}
\caption{Comparisons of mean computation time of the RGAS and RAGAS algorithms on CPU (log-log scale) for different learning rates. The shaded areas show the max-min values over 50 runs.}\label{fig:exp4_computation_time_lr_above}
\end{wrapfigure}

\paragraph{Computation time of algorithms.} We conduct our third experiment on the fragmented hypercube with dimension $d \in \{25, 50, 100, 250, 500\}$, subspace dimension $k=2$, number of points $n=100$ and threshold $\epsilon=0.001$. For the SRW and the PRW distances, the regularization parameter is set as $\eta = 0.2$ for $n < 250$ and $\eta = 0.5$ otherwise\footnote{Available in https://github.com/francoispierrepaty/SubspaceRobustWasserstein}, as well as the scaling for the matrix $C$ (cf. Definition~\ref{def:C}) is applied for stabilizing the algorithms. We stop the RGAS and RAGAS algorithms when $\|U_{t+1} - U_t\|_F/\|U_t\|_F \leq \epsilon$. 

Figure~\ref{fig:exp_noise_2} (right) presents the mean computation time of the SRW distance with the Frank-Wolfe algorithm~\citep{Paty-2019-Subspace} and the PRW distance with our RGAS and RAGAS algorithms. Our approach is significantly faster since the complexity bound of their approach is quadratic in dimension $d$ while our methods are linear in dimension $d$. 

\begin{figure}[!t]
\centering
\includegraphics[clip, trim=10 20 10 20, width=0.45\textwidth]{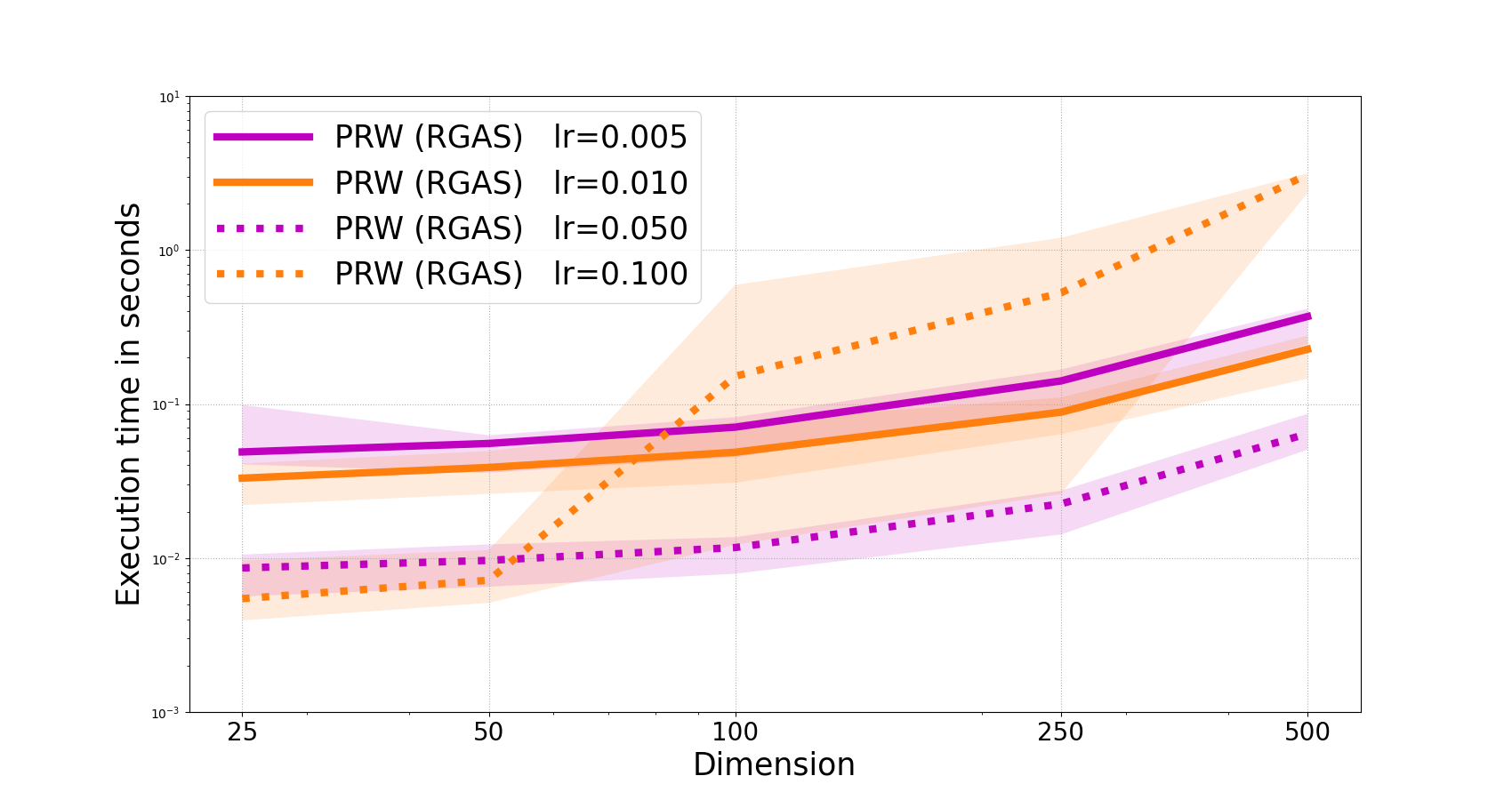}
\includegraphics[clip, trim=20 10 10 20, width=0.45\textwidth]{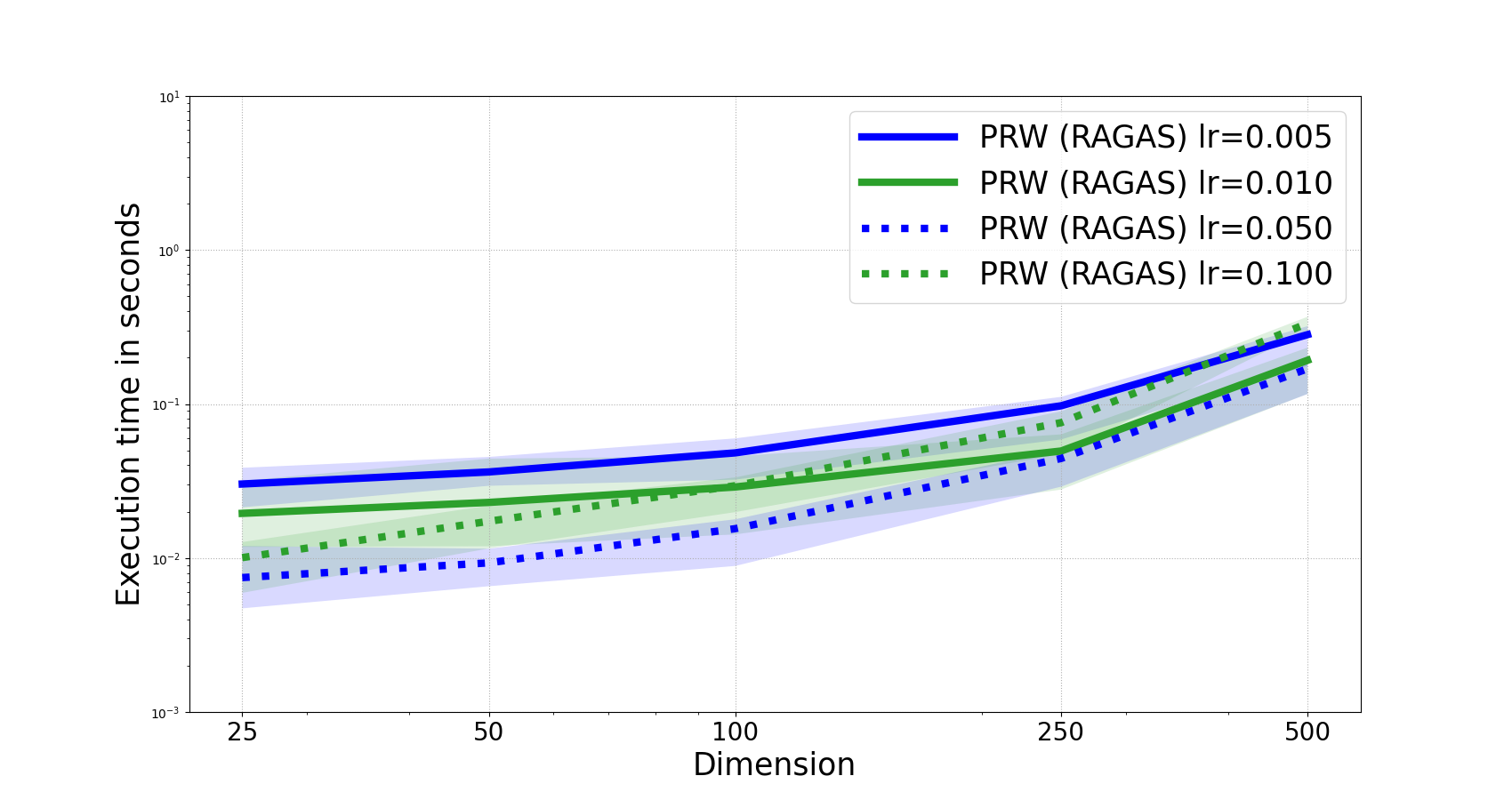}
\caption{Comparisons of mean computation time of the RGAS and RAGAS algorithms on CPU (log-log scale) for different learning rates. The shaded areas show the max-min values over 50 runs.}\label{fig:exp4_computation_time_lr_bottom}
\end{figure}
\paragraph{Robustness of algorithms to learning rate.} We conduct our fourth experiment on the fragmented hypercube to evaluate the robustness of our RGAS and RAGAGS algorithms by choosing the learning rate $\gamma \in \{0.01, 0.1\}$. The parameter setting is the same as that in the third experiment.

Figure~\ref{fig:exp4_computation_time_lr_above} indicates that the RAGAS algorithm is more robust than the RGAS algorithm as the learning rates varies, with smaller variance in computation time (seconds). This is the case especially when the dimension is large, showing the advantage of the adaptive strategies in practice. 

To demonstrate the advantage of the adaptive strategies in practice, we initialize the learning rate using four options $\gamma \in \{0.005, 0.01, 0.05, 0.1\}$ and present the results for the RGAS and RAGAS algorithms separately in Figure~\ref{fig:exp4_computation_time_lr_bottom}. This is consistent with the results in Figure~\ref{fig:exp4_computation_time_lr_above} and supports that the RAGAS algorithm is more robust than the RGAS algorithm to the learning rate. 
\begin{table}[!t]\small
\centering\hspace*{-3.5em}
\begin{tabular}{|c|ccccccc|} \hline
& D & G & I & KB1 & KB2 & TM & T \\ \hline
D & 0/0 & 0.184/0.126 & 0.185/0.135 & 0.195/0.153 & 0.202/0.162 & 0.186/0.134 & \textbf{0.170}/\textbf{0.105} \\ \hline
G & 0.184/0.126 & 0/0 & \textbf{0.172}/0.101 & 0.196/0.146 & 0.203/0.158 & 0.175/\textbf{0.095} & 0.184/0.128 \\ \hline
I & 0.185/0.135 & 0.172/0.101 & 0/0 & 0.195/0.155 & 0.203/0.166 & \textbf{0.169}/\textbf{0.099} & 0.180/0.134 \\ \hline
KB1 & 0.195/0.153 & 0.196/0.146 & 0.195/0.155 & 0/0 & \textbf{0.164}/\textbf{0.089} & 0.190/0.146 & 0.179/0.132 \\ \hline
KB2 & 0.202/0.162 & 0.203/0.158 & 0.203/0.166 & \textbf{0.164}/\textbf{0.089} & 0/0 & 0.193/0.155 & 0.180/0.138 \\ \hline
TM & 0.186/0.134 & 0.175/\textbf{0.095} & \textbf{0.169}/0.099 & 0.190/0.146 & 0.193/0.155 & 0/0 & 0.182/0.136 \\ \hline
T & \textbf{0.170}/\textbf{0.105} & 0.184/0.128 & 0.180/0.134 & 0.179/0.132 & 0.180/0.138 & 0.182/0.136 & 0/0 \\ \hline
\end{tabular}
\caption{Each entry is $\SCal_k^2/\PCal_k^2$ distance between different movie scripts. D = Dunkirk, G = Gravity, I = Interstellar, KB1 = Kill Bill Vol.1, KB2 = Kill Bill Vol.2, TM = The Martian, T = Titanic.}\label{tab:exp_cinema}
\end{table}
\begin{table}[!t]\small
\centering
\begin{tabular}{|c|cccccc|}\hline
& H5 & H & JC  & TMV & O & RJ \\ \hline
H5 & 0/0 & \textbf{0.222}/\textbf{0.155} & 0.230/0.163 & 0.228/0.166 & 0.227/0.170 & 0.311/0.272 \\ \hline
H & 0.222/0.155 & 0/0 & 0.224/0.163 & 0.221/0.159 & \textbf{0.220}/\textbf{0.153} & 0.323/0.264 \\ \hline
JC & 0.230/0.163 & 0.224/0.163 & 0/0 & 0.221/\textbf{0.156} & \textbf{0.219}/0.157 & 0.246/0.191 \\ \hline
TMV & 0.228/0.166 & 0.221/0.159 & \textbf{0.221}/0.156 & 0/0 & 0.222/\textbf{0.154} & 0.292/0.230 \\ \hline
O & 0.227/0.170 & 0.220/\textbf{0.153} & \textbf{0.219}/0.157 & 0.222/0.154 & 0/0 & 0.264/0.215 \\ \hline
RJ & 0.311/0.272 & 0.323/0.264 & \textbf{0.246}/\textbf{0.191} & 0.292/0.230 & 0.264/0.215 & 0/0 \\ \hline
\end{tabular}
\caption{Each entry is $\SCal_k^2/\PCal_k^2$ distance between different \sk plays. H5 = Henry V, H = Hamlet, JC = Julius Caesar, TMV = The Merchant of Venice, O = Othello, RJ = Romeo and Juliet.}\label{tab:exp_shakespeare}
\end{table}

\paragraph{Experiments on real data.} We compute the PRW and SRW distances between all pairs of items in a corpus of seven \textit{movie scripts}. Each script is tokenized to a list of words, which is transformed to a measure over $\br^{300}$ using \textsc{word2vec}~\citep{Mikolov-2018-Advances} where each weight is word frequency. The SRW and PRW distances between all pairs of movies are in Table~\ref{tab:exp_cinema}, which is consistent with the SRW distance in~\citet[Figure~9]{Paty-2019-Subspace} and demonstrate that the PRW distance is smaller than SRW distance. We also compute the SRW and PRW for a preprocessed corpus of eight Shakespeare operas. The PRW distance is consistently smaller than the corresponding SRW distance; see Table~\ref{tab:exp_shakespeare}. Figure~\ref{fig:exp_word_cloud} displays the projection of two measures associated with \textit{Dunkirk} versus \textit{Interstellar} (left) and \textit{Julius Caesar} versus \textit{The Merchant of Venice} (right) onto their optimal 2-dimensional projection.

To further show the versatility of SRW and PRW distances, we extract the features of different MNIST digits using a convolutional neural network (CNN) and compute the scaled SRW and PRW distances between all pairs of MNIST digits. In particular, we use an off-the-shelf PyTorch implementation\footnote{https://github.com/pytorch/examples/blob/master/mnist/main.py} and pretrain on MNIST with 98.6\% classification accuracy on the test set. We extract the 128-dimensional features of each digit from the penultimate layer of the CNN. Since the MNIST test set contains 1000 images per digit, each digit is associated with a measure over $\br^{128000}$. Then we compute the optimal 2-dimensional projection distance of measures associated with each pair of two digital classes and divide each distance by 1000; see Table~\ref{tab:MNIST} for the details. The minimum SRW and PRW distances in each row is highlighted to indicate its most similar digital class of that row, which coincides with our intuitions. For example, D1 is sometimes confused with D7 (0.58/0.47), while D4 is often confused with D9 (0.49/0.38) in scribbles.

\paragraph{Summary.} The PRW distance has less discriminative power than the SRW distance which is equivalent to the Wasserstein distance~\citep[Proposition~2]{Paty-2019-Subspace}. Such equivalence implies that the SRW distance suffers from the curse of dimensionality in theory. In contrast, the PRW distance has much better sample complexity than the SRW distance if the distributions satisfy the mild condition~\citep{Niles-2019-Estimation, Lin-2021-Projection}. Our empirical evaluation shows that the PRW distance is computationally favorable and more robust than the SRW and Wasserstein distance, when the noise has the moderate to high variance.  
\begin{table}[!t]\small
\centering\hspace*{-6em}
\begin{tabular}{|c|cccccccccc|}
\hline
& D0 & D1 & D2 & D3 & D4 & D5 & D6 & D7 & D8 & D9 \\ \hline
D0 & 0/0 & 0.97/0.79 & \textbf{0.80}/\textbf{0.59} & 1.20/0.92 & 1.23/0.90 & 1.03/0.71 & 0.81/0.59 & 0.86/0.66 & 1.06/0.79 & 1.09/0.81 \\ \hline
D1 & 0.97/0.79 & 0/0 & 0.66/0.51 & 0.86/0.72 & 0.68/0.54 & 0.84/0.70 & 0.80/0.66 & \textbf{0.58}/\textbf{0.47} & 0.88/0.71 & 0.85/0.72 \\ \hline
D2 & 0.80/0.59 & \textbf{0.66}/\textbf{0.51} & 0/0 & 0.73/0.54 & 1.08/0.79 & 1.08/0.83 & 0.90/0.70 & 0.70/0.53 & 0.68/0.52 & 1.07/0.81 \\ \hline
D3 & 1.20/0.92 & 0.86/0.72 & 0.73/0.54 & 0/0 & 1.20/0.87 & \textbf{0.58}/\textbf{0.43} & 1.23/0.91 & 0.72/0.55 & 0.88/0.64 & 0.83/0.65 \\ \hline
D4 & 1.23/0.90 & 0.68/0.54 & 1.08/0.79 & 1.20/0.87 & 0/0 & 1.00/0.75 & 0.85/0.62 & 0.79/0.61 & 1.09/0.78 & \textbf{0.49}/\textbf{0.38} \\ \hline
D5 & 1.03/0.71 & 0.84/0.70 & 1.08/0.83 & \textbf{0.58}/\textbf{0.43} & 1.00/0.75 & 0/0 & 0.72/0.51 & 0.91/0.68 & 0.72/0.53 & 0.78/0.59 \\ \hline
D6 & 0.81/0.59 & 0.80/0.66 & 0.90/0.70 & 1.23/0.91 & 0.85/0.62 & \textbf{0.72}/\textbf{0.51} & 0/0 & 1.11/0.83 & 0.92/0.66 & 1.22/0.83 \\ \hline
D7 & 0.86/0.66 & \textbf{0.58}/0.47 & 0.70/0.53 & 0.72/0.55 & 0.79/0.61 & 0.91/0.68 & 1.11/0.83 & 0/0 & 1.07/0.78 & 0.62/\textbf{0.46} \\ \hline
D8 & 1.06/0.79 & 0.88/0.71 & \textbf{0.68}/\textbf{0.52} & 0.88/0.64 & 1.09/0.78 & 0.72/0.53 & 0.92/0.66 & 1.07/0.78 & 0/0 & 0.87/0.63 \\ \hline
D9 & 1.09/0.81 & 0.85/0.72 & 1.07/0.81 & 0.83/0.65 & \textbf{0.49}/\textbf{0.38} & 0.78/0.59 & 1.22/0.83 & 0.62/0.46 & 0.87/0.63 & 0/0 \\ \hline
\end{tabular}
\caption{Each entry is scaled $\SCal_k^2/\PCal_k^2$ distance between different hand-written digits.}\label{tab:MNIST}
\end{table}
\begin{figure}[!t]
\centering
\includegraphics[clip, trim=0 0 0 0, width=1\textwidth]{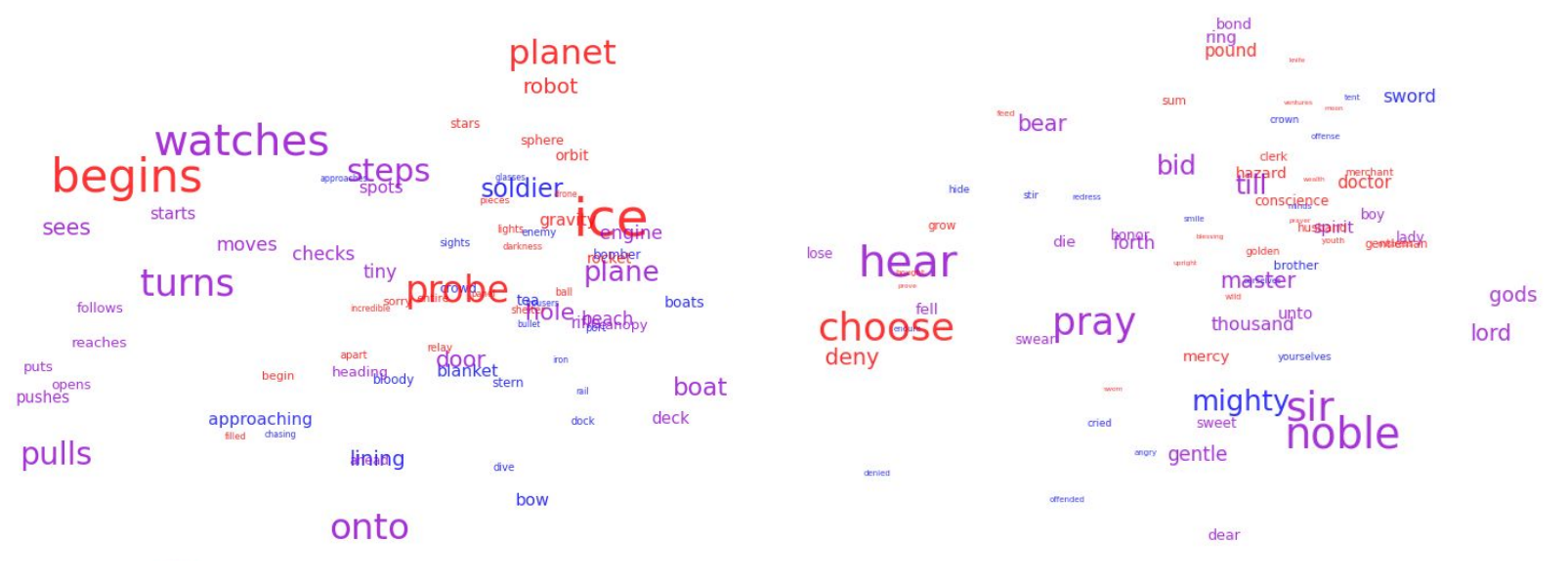}
\caption{Optimal 2-dimensional projections between ``Dunkirk" and ``Interstellar" (left) and optimal 2-dimensional projections between ``Julius Caesar" and ``The Merchant of Venice" (right). Common words of two items are displayed in violet and the 30 most frequent words of each item are displayed.}\label{fig:exp_word_cloud}\vspace*{-.5em}
\end{figure}

\section{Conclusion}\label{sec:conclusions}
We study in this paper the computation of the projection robust Wasserstein (PRW) distance in the discrete setting. A set of algorithms are developed for computing the entropic regularized PRW distance and both guaranteed to converge to an approximate pair of optimal subspace projection and optimal transportation plan. Experiments on synthetic and real datasets demonstrate that our approach to computing the PRW distance is an improvement over existing approaches based on the convex relaxation of the PRW distance and the Frank-Wolfe algorithm. Future work includes the theory for continuous distributions and applications of PRW distance to deep generative models.    

\section{Acknowledgments}
We would like to thank four anonymous referees for constructive suggestions that improve the quality of this paper. This work is supported in part by the Mathematical Data Science program of the Office of Naval Research under grant number N00014-18-1-2764.

\bibliographystyle{plainnat}
\bibliography{ref}

\newpage\appendix
\section{Further Background Materials on Riemannian Optimization}
The problem of optimizing a smooth function over the Riemannian manifold has been the subject of a large literature.~\citet{Absil-2009-Optimization} provide a comprehensive treatment, showing how first-order and second-order algorithms are extended to the Riemannian setting and proving asymptotic convergence to first-order stationary points.~\citet{Boumal-2019-Global} have established global sublinear convergence results for Riemannian gradient descent and Riemannian trust region algorithms, and further showed that the latter approach converges to a second order stationary point in polynomial time; see also~\citet{Kasai-2018-Inexact, Hu-2018-Adaptive, Hu-2019-Structured}. In contradistinction to the Euclidean setting, the Riemannian trust region algorithm requires a Hessian oracle. There have been also several recent papers on problem-specific algorithms~\citep{Wen-2013-Feasible, Gao-2018-New, Liu-2019-Quadratic} and primal-dual algorithms~\citep{Zhang-2019-Primal} for Riemannian optimization.

Compared to the smooth setting, Riemannian nonsmooth optimization is harder and relatively less explored~\citep{Absil-2019-Collection}. There are two main lines of work. In the first category, one considers optimizing geodesically convex function over a Riemannian manifold with subgradient-type algorithms; see, e.g.,~\citet{Ferreira-1998-Subgradient, Zhang-2016-First, Bento-2017-Iteration}. In particular,~\citet{Ferreira-1998-Subgradient} first established an asymptotic convergence result while~\citet{Zhang-2016-First, Bento-2017-Iteration} derived a global convergence rate of $\bigO(\epsilon^{-2})$ for the Riemannian subgradient algorithm. Unfortunately, these results are not useful for understanding the computation of the PRW distance in Eq.~\eqref{prob:main} since the Stiefel manifold is \textit{compact} and every continuous and geodesically convex function on a compact Riemannian manifold must be a constant; see~\citet[Proposition~2.2]{Bishop-1969-Manifolds}. In the second category, one assumes the tractable computation of the proximal mapping of the objective function over the Riemannian manifold.~\citet{Ferreira-2002-Proximal} proved that the Riemannian proximal point algorithm converges globally at a sublinear rate. 

When specialized to the Stiefel manifold,~\citet{Chen-2020-Proximal} consider the composite objective and proposed to compute the proximal mapping of nonsmooth component function over the tangent space. The resulting Riemannian proximal gradient algorithm is practical in real applications while achieving theoretical guarantees.~\citet{Li-2019-Nonsmooth} extended the results in~\citet{Davis-2019-Stochastic} to the Riemannian setting and proposed a family of Riemannian subgradienttype methods for optimizing a weakly convex function over the Stiefel manifold. They also proved that their algorithms have an iteration complexity of $\bigO(\epsilon^{-4})$ for driving a near-optimal stationarity measure below $\epsilon$. Following up the direction proposed by~\citet{Li-2019-Nonsmooth}, we derive a near-optimal condition (Definition~\ref{def:near-stationarity} and~\ref{def:near-optimal-pair}) for the max-min optimization model in Eq.~\eqref{prob:main-CCP} and propose an algorithm with the finite-time convergence under this stationarity measure. 

Finally, there are several results on stochastic optimization over the Riemannian manifold.~\citet{Bonnabel-2013-Stochastic} proved the first asymptotic convergence result for Riemannian stochastic gradient descent, which is further extended by~\citet{Zhang-2016-Riemannian, Tripuraneni-2018-Averaging, Becigneul-2019-Riemannian}. If the Riemannian Hessian is not positive definite, a few recent works have developed frameworks to escape saddle points~\citep{Sun-2019-Escaping, Criscitiello-2019-Efficiently}. 

\section{Near-Optimality Condition}
In this section, we derive a near-optimal condition (Definition~\ref{def:near-stationarity} and~\ref{def:near-optimal-pair}) for the max-min optimization model in Eq.~\eqref{prob:main} and the maximization of $f$ over $\St(d, k)$ in Eq.~\eqref{prob:Stiefel-nonsmooth}. Following~\citet{Davis-2019-Stochastic, Li-2019-Nonsmooth}, we define the proximal mapping of $f$ over $\St(d, k)$ in Eq.~\eqref{prob:Stiefel-nonsmooth}, which takes into account both the Stiefel manifold constraint and max-min structure\footnote{The proximal mapping $p(U)$ must exist since the Stiefel manifold is compact, yet may not be uniquely defined. However, this does not matter since $p(U)$ only appears in the analysis for the purpose of defining the surrogate stationarity measure; see~\citet{Li-2019-Nonsmooth}.}: 
\begin{equation*}
p(U) \ \in \ \argmax\limits_{\bar{U} \in \St(d, k)} \ \left\{f(\bar{U}) - 6\|C\|_\infty\|\bar{U} - U\|_F^2\right\} \quad \text{for all } U \in \St(d, k). 
\end{equation*}
After a simple calculation, we have
\begin{equation*}
\Theta(U) \ \mydefn \ 12\|C\|_\infty\|p(U) - U\|_F \ \geq \ \dist(0, \subdiff f(\prox_{\rho f}(U))), 
\end{equation*}
Therefore, we conclude from Definition~\ref{def:stationarity} that $p(U) \in \St(d, k)$ is $\epsilon$-approximate optimal subspace projection of $f$ over $\St(d, k)$ in Eq.~\eqref{prob:Stiefel-nonsmooth} if $\Theta(U) \leq \epsilon$. We remark that $\Theta(\bullet)$ is a well-defined surrogate stationarity measure of $f$ over $\St(d, k)$ in Eq.~\eqref{prob:Stiefel-nonsmooth}. Indeed, if $\Theta(U) = 0$, then $U \in \St(d, k)$ is an optimal subspace projection. This inspires the following $\epsilon$-near-optimality condition for any $\widehat{U} \in \St(d, k)$. 
\begin{definition}\label{def:near-stationarity}
A subspace projection $\widehat{U} \in \St(d, k)$ is called an \emph{$\epsilon$-approximate near-optimal subspace projection} of $f$ over $\St(d, k)$ in Eq.~\eqref{prob:Stiefel-nonsmooth} if it satisfies $\Theta(\widehat{U}) \leq \epsilon$. 
\end{definition}
Equipped with Definition~\ref{def:approx_transportation_plan} and~\ref{def:near-stationarity}, we define an $\epsilon$-approximate pair of near-optimal subspace projection and optimal transportation plan for the computation of the PRW distance in Eq.~\eqref{prob:main}. 
\begin{definition}\label{def:near-optimal-pair}
The pair of subspace projection and transportation plan $(\widehat{U}, \widehat{\pi}) \in \St(d, k) \times \Pi(\mu, \nu)$ is an \emph{$\epsilon$-approximate pair of near-optimal subspace projection and optimal transportation plan} for the computation of the PRW distance in Eq.~\eqref{prob:main} if the following statements hold true: 
\begin{itemize}
\item $\widehat{U}$ is an $\epsilon$-approximate near-optimal subspace projection of $f$ over $\St(d, k)$ in Eq.~\eqref{prob:Stiefel-nonsmooth}. 
\item $\widehat{\pi}$ is an $\epsilon$-approximate optimal transportation plan for the subspace projection $\widehat{U}$. 
\end{itemize}
\end{definition}
Finally, we prove that the stationary measure in Definition~\ref{def:near-optimal-pair} is a local surrogate for the stationary measure in Definition~\ref{def:optimal-pair} in the following proposition. 
\begin{proposition}
If $(U, \pi) \in \St(d, k) \times \Pi(\mu, \nu)$ is an $\epsilon$-approximate pair of optimal subspace projection and optimal transportation plan of problem~\eqref{prob:main}, it is an $3\epsilon$-approximate pair of optimal subspace projection and optimal transportation plan. 
\end{proposition}
\begin{proof}
By the definition, $(U, \pi) \in \St(d, k) \times \Pi(\mu, \nu)$ satisfies that $\pi$ is an $\epsilon$-approximate optimal transportation plan for the subspace projection $U$. Thus, it suffices to show that $\Theta(U) \leq 3\epsilon$. By the definition of $p(U)$, we have
\begin{equation*}
f(p(U)) - 6\|C\|_\infty\|p(U) - U\|_F^2 \ \geq \ f(U). 
\end{equation*}
Since $f$ is $2\|C\|_\infty$-weakly concave and each element of the subdifferential $\partial f(U)$ is bounded by $2\|C\|_\infty$ for all $U \in \St(d, k)$, the Riemannian subgradient inequality~\citep[Theorem~1]{Li-2019-Nonsmooth} implies that 
\begin{equation*}
f(\prox_{\rho f}(U)) - f(U) \ \leq \ \langle \xi, \prox_{\rho f}(U) - U\rangle + 2\|C\|_\infty\|\prox_{\rho f}(U) - U\|^2 \quad \text{for any } \xi \in \subdiff f(U). 
\end{equation*}
Since $\dist(0, \subdiff f(U)) \leq \epsilon$, we have
\begin{equation*}
f(\prox_{\rho f}(U)) - f(U) \ \leq \ \epsilon\|\prox_{\rho f}(U) - U\|_F + 2\|C\|_\infty\|\prox_{\rho f}(U) - U\|^2. 
\end{equation*}
Putting these pieces together with the definition of $\Theta(U)$ yields the desired result. 
\end{proof}

\section{Riemannian Supergradient meets Network Simplex Iteration}
In this section, we propose a new algorithm, named \emph{Riemannian SuperGradient Ascent with Network simplex iteration} (RSGAN), for computing the PRW distance in Eq.~\eqref{prob:main}. The iterates are guaranteed to converge to an $\epsilon$-approximate pair of \textit{near-optimal} subspace projection and optimal transportation plan (cf. Definition~\ref{def:near-optimal-pair}). The complexity bound is $\bigOtil(n^2(d + n)\epsilon^{-4})$ if $k = \bigOtil(1)$.
\begin{algorithm}[!t]
\caption{Riemannian SuperGradient Ascent with Network Simplex Iteration (RSGAN)}\label{alg:supergrad-simplex}
\begin{algorithmic}[1]
\STATE \textbf{Input:} measures $\{(x_i, r_i)\}_{i \in [n]}$ and $\{(y_j, c_j)\}_{j \in [n]}$, dimension $k = \bigOtil(1)$ and tolerance $\epsilon$.  
\STATE \textbf{Initialize:} $U_0 \in \St(d, k)$, $\widehat{\epsilon} \leftarrow \frac{\epsilon}{10\|C\|_\infty}$ and $\gamma_0 \leftarrow \frac{1}{k\|C\|_\infty}$. 
\FOR{$t = 0, 1, 2, \ldots, T-1$}
\STATE Compute $\pi_{t+1} \leftarrow \textsc{OT}(\{(x_i, r_i)\}_{i \in [n]}, \{(y_j, c_j)\}_{j \in [n]}, U_t, \widehat{\epsilon})$.
\STATE Compute $\xi_{t+1} \leftarrow P_{\Tg_{U_t}\St}(2V_{\pi_{t+1}}U_t)$. 
\STATE Compute $\gamma_{t+1} \leftarrow \gamma_0/\sqrt{t+1}$.
\STATE Compute $U_{t+1} \leftarrow \retr_{U_t}(\gamma_{t+1}\xi_{t+1})$.  
\ENDFOR
\end{algorithmic}
\end{algorithm}
\subsection{Algorithmic scheme}
We start with a brief overview of the Riemannian supergradient ascent algorithm for nonsmooth Stiefel optimization. Letting $F: \br^{d \times k} \rightarrow \br$ be a nonsmooth but weakly concave function, we consider 
\begin{equation*}
\max\limits_{U \in \St(d, k)} \ F(U). 
\end{equation*}
A generic Riemannian supergradient ascent algorithm for solving this problem is given by
\begin{equation*}
U_{t+1} \ \leftarrow \ \retr_{U_t}(\gamma_{t+1}\xi_{t+1}) \quad \textnormal{ for any } \xi_{t+1} \in \subdiff F(U_t),   
\end{equation*} 
where $\subdiff F(U_t)$ is Riemannian subdifferential of $F$ at $U_t$ and $\retr$ is any retraction on $\St(d, k)$. For the nonconvex nonsmooth optimization, the stepsize setting $\gamma_{t+1} = \gamma_0/\sqrt{t+1}$ is widely accepted in both theory and practice~\citep{Davis-2019-Stochastic, Li-2019-Nonsmooth}. 

By the definition of Riemannian subdifferential, $\xi_t$ can be obtained by taking $\xi \in \partial F(U)$ and by setting $\xi_t = P_{\Tg_U\St}(\xi)$. Thus, it is necessary for us to specify the subdifferential of $f$ in Eq.~\eqref{prob:Stiefel-nonsmooth}. Using the symmetry of $V_\pi$, we have
\begin{equation*}
\partial f(U) \ = \ \textnormal{Conv}\left\{2V_{\pi^\star} U \mid \pi^\star \in \argmin\limits_{\pi \in \Pi(\mu, \nu)} \ \langle UU^\top, V_\pi\rangle\right\}, \quad \textnormal{ for any } U \in \br^{d \times k}.  
\end{equation*}
The remaining step is to solve an OT problem with a given $U$ at each inner loop of the maximization and use the output $\pi(U)$ to obtain an inexact supergradient of $f$. Since the OT problem with a given $U$ is exactly an LP, this is possible and can be done by applying the variant of network simplex method in the \textsc{POT} package~\citep{Flamary-2017-Pot}. While the simplex method can exactly solve this LP, we adopt the inexact solving rule as a practical matter. More specifically, the output $\pi_{t+1}$ satisfies that $\pi_{t+1} \in \Pi(\mu, \nu)$ and $\|\pi_{t+1} - \pi_t^\star\|_1 \leq \widehat{\epsilon}$ where $\pi_t^\star$ is an optimal solution of unregularized OT problem with $U_t \in \St(d, k)$. With the inexact solving rule, the interior-point method and some first-order methods can be adopted to solve the unregularized OT problem. To this end, we summarize the pseudocode of the RSGAN algorithm in Algorithm~\ref{alg:supergrad-simplex}.

\subsection{Complexity analysis for Algorithm~\ref{alg:supergrad-simplex}}
We define a function which is important to the subsequent analysis of Algorithm~\ref{alg:supergrad-simplex}: 
\begin{equation*}
\Phi(U) \ \mydefn \ \max\limits_{U' \in \St(d, k)} \ \left\{f(U') - 6\|C\|_\infty\|U' - U\|_F^2\right\} \quad \text{for all } U \in \St(d, k). 
\end{equation*}
Our first lemma provides a key inequality for quantifying the progress of the iterates $\{(U^t, \pi^t)\}_{t \geq 1}$ generated by Algorithm~\ref{alg:supergrad-simplex} using $\Phi(\bullet)$ as the potential function.
\begin{lemma}\label{lemma:key-descent-supergrad}
Letting $\{(U_t, \pi_t)\}_{t \geq 1}$ be the iterates generated by Algorithm~\ref{alg:supergrad-simplex}, we have
\begin{eqnarray*}
\Phi(U_{t+1}) & \geq & \Phi(U_t) - 12\gamma_{t+1}\|C\|_\infty\left(f(U_t) - f(p(U_t)) + 4\|C\|_\infty\|p(U_t) - U_t\|_F^2 + \frac{\epsilon^2}{200\|C\|_\infty}\right) \\
& - & 200\gamma_{t+1}^2\|C\|_\infty^3(\gamma_{t+1}^2 L_2^2\|C\|_\infty^2 + \gamma_{t+1}\|C\|_\infty + \sqrt{k}). \nonumber
\end{eqnarray*}
\end{lemma}
\begin{proof}
Since $p(U_t) \in \St(d, k)$, we have
\begin{equation}\label{inequality:obj-progress-first}
\Phi(U_{t+1}) \ \geq \ f(p(U_t)) - 6\|C\|_\infty\|p(U_t) - U_{t+1}\|_F^2. 
\end{equation}
Using the update formula of $U_{t+1}$, we have 
\begin{equation*}
\|p(U_t) - U_{t+1}\|_F^2 \ = \ \|p(U_t) - \retr_{U_t}(\gamma_{t+1}\xi_{t+1})\|_F^2.  
\end{equation*}
Using the Cauchy-Schwarz inequality and Proposition~\ref{prop:retraction}, we have
\begin{eqnarray*}
\lefteqn{\|p(U_t) - \retr_{U_t}(\gamma_{t+1}\xi_{t+1})\|_F^2}  \\ 
& = & \|(U_t + \gamma_{t+1}\xi_{t+1} - p(U_t)) + (\retr_{U_t}(\gamma_{t+1}\xi_{t+1}) - U_t - \gamma_{t+1}\xi_{t+1})\|_F^2 \\
& \leq & \|U_t + \gamma_{t+1}\xi_{t+1} - p(U_t)\|_F^2 + \|\retr_{U_t}(\gamma_{t+1}\xi_{t+1}) - (U_t + \gamma_{t+1}\xi_{t+1})\|_F^2 \\ 
& & + 2\|U_t + \gamma_{t+1}\xi_{t+1} - p(U_t)\|_F\|\retr_{U_t}(\gamma_{t+1}\xi_{t+1}) - (U_t + \gamma_{t+1}\xi_{t+1})\|_F \\
& \leq & \|U_t + \gamma_{t+1}\xi_{t+1} - p(U_t)\|_F^2 + \gamma_{t+1}^4 L_2^2\|\xi_{t+1}\|_F^4 + 2\gamma_{t+1}^2\|U_t + \gamma_{t+1}\xi_{t+1} - p(U_t)\|_F\|\xi_{t+1}\|_F^2 \\
& \leq & \|U_t - p(U_t)\|_F^2 + 2\gamma_{t+1}\langle\xi_{t+1}, U_t - p(U_t)\rangle + \gamma_{t+1}^2\|\xi_{t+1}\|_F^2 + \gamma_{t+1}^4 L_2^2\|\xi_{t+1}\|_F^4 \\ 
& & + 2\gamma_{t+1}^2\|U_t + \gamma_{t+1}\xi_{t+1} - p(U_t)\|_F\|\xi_{t+1}\|_F^2. 
\end{eqnarray*}
Since $U_t \in \St(d, k)$ and $p(U_t) \in \St(d, k)$, we have $\|U_t\|_F \leq \sqrt{k}$ and $\|p(U_t)\|_F \leq \sqrt{k}$. By the update formula for $\xi_{t+1}$, we have
\begin{equation*}
\|\xi_{t+1}\|_F \ = \ \|P_{\Tg_{U_{t-1}}\St}(2V_{\pi_{t+1}}U_t)\|_F \ \leq \ 2\|V_{\pi_{t+1}}U_t\|_F. 
\end{equation*}
Since $U_t \in \St(d, k)$ and $\pi_{t+1} \in \Pi(\mu, \nu)$, we have $\|\xi_{t+1}\|_F \leq 2\|C\|_\infty$. Putting all these pieces together yields that 
\begin{eqnarray}\label{inequality:obj-progress-second}
\|p(U_t) - U_{t+1}\|_F^2 & \leq & \|U_t - p(U_t)\|_F^2 + 2\gamma_{t+1}\langle\xi_{t+1}, U_t - p(U_t)\rangle + 4\gamma_{t+1}^2\|C\|_\infty^2 \\ 
& & \hspace*{-4em} + 16\gamma_{t+1}^4 L_2^2\|C\|_\infty^4 + 16\gamma_{t+1}^3\|C\|_\infty^3 + 16\gamma_{t+1}^2\sqrt{k}\|C\|_\infty^2. \nonumber
\end{eqnarray}
Plugging Eq.~\eqref{inequality:obj-progress-second} into Eq.~\eqref{inequality:obj-progress-first} and simplifying the inequality using $k \geq 1$, we have 
\begin{eqnarray*}
\Phi(U_{t+1}) & \geq & f(p(U_t)) - 6\|C\|_\infty\|U_t - p(U_t)\|_F^2 - 12\gamma_{t+1}\|C\|_\infty\langle\xi_{t+1}, U_t - p(U_t)\rangle \\
& & - 200\gamma_{t+1}^2\|C\|_\infty^3\left(\gamma_{t+1}^2 L_2^2\|C\|_\infty^2 + \gamma_{t+1}\|C\|_\infty + \sqrt{k}\right). 
\end{eqnarray*}
By the definition of $\Phi(\bullet)$ and $p(\bullet)$, we have
\begin{eqnarray}\label{inequality:obj-progress-third}
\Phi(U_{t+1}) & \geq & \Phi(U_t) - 12\gamma_{t+1}\|C\|_\infty\langle\xi_{t+1}, U_t - p(U_t)\rangle \\
& & \hspace*{-4em} - 200\gamma_{t+1}^2\|C\|_\infty^3\left(\gamma_{t+1}^2 L_2^2\|C\|_\infty^2 + \gamma_{t+1}\|C\|_\infty + \sqrt{k}\right). \nonumber
\end{eqnarray}
Now we proceed to bound the term $\langle\xi_{t+1}, U_t - p(U_t)\rangle$. Letting $\xi_t^\star = P_{\Tg_{U_t}\St}(2V_{\pi_t^\star}U_t)$ where $\pi_t^\star$ is a minimizer of unregularized OT problem, i.e., $\pi_t^\star \in \argmin_{\pi \in \Pi(\mu, \nu)} \ \langle U_tU_t^\top, V_\pi\rangle$, we have
\begin{equation}\label{inequality:obj-progress-fourth}
\langle\xi_{t+1}, U_t - p(U_t)\rangle \ \leq \ \langle\xi_t^\star, U_t - p(U_t)\rangle + \|\xi_{t+1} - \xi_t^\star\|_F\|U_t - p(U_t)\|_F. 
\end{equation}
Since $f(U) = \min_{\pi \in \Pi(\mu, \nu)} \ \langle U_tU_t^\top, V_\pi\rangle$ is $2\|C\|_\infty$-weakly concave over $\br^{d \times k}$ (cf. Lemma~\ref{lemma:obj-weak-concave}), $\xi_t^\star \in \subdiff f(U_t)$ and each element in the subdifferential $\partial f(U)$ is bounded by $2\|C\|_\infty$ for all $U \in \St(d, k)$ (cf. Lemma~\ref{lemma:bound-subdiff}), the Riemannian subgradient inequality~\citep[Theorem~1]{Li-2019-Nonsmooth} holds true and implies that 
\begin{equation*}
f(p(U_t)) \ \leq \ f(U_t) + \langle\xi_t^\star, p(U_t) - U_t\rangle + 2\|C\|_\infty\|p(U_t) - U_t\|_F^2. 
\end{equation*}
This implies that 
\begin{equation}\label{inequality:obj-progress-fifth}
\langle\xi_t^\star, U_t - p(U_t)\rangle \ \leq \ f(U_t) - f(p(U_t)) + 2\|C\|_\infty\|p(U_t) - U_t\|_F^2. 
\end{equation}
By the definition of $\xi_{t+1}$ and $\xi_t^\star$, we have
\begin{equation*}
\|\xi_{t+1} - \xi_t^\star\|_F \ = \ \|P_{\Tg_{U_t}\St}(2V_{\pi_{t+1}}U_t) - P_{\Tg_{U_t}\St}(2V_{\pi_t^\star}U_t)\|_F \ \leq \ 2\|(V_{\pi_{t+1}} - V_{\pi_t^\star})U_t\|_F. 
\end{equation*}
By the definition of the subroutine $\textsc{OT}(\{(x_i, r_i)\}_{i \in [n]}, \{(y_j, c_j)\}_{j \in [n]}, U, \widehat{\epsilon})$ in Algorithm~\ref{alg:supergrad-simplex}, we have $\pi_{t+1} \in \Pi(\mu, \nu)$ and $\|\pi_{t+1} - \pi_t^\star\|_1 \leq \widehat{\epsilon}$. Thus, we have
\begin{equation*}
\|\xi_{t+1} - \xi_t^\star\|_F \ \leq \ 2\|C\|_\infty\widehat{\epsilon} \ \leq \ \frac{\epsilon}{5}.  
\end{equation*}
Using Young's inequality, we have
\begin{eqnarray}\label{inequality:obj-progress-sixth}
\|\xi_{t+1} - \xi_t^\star\|_F\|U_t - p(U_t)\|_F & \leq & \frac{\|\xi_{t+1} - \xi_t^\star\|_F^2}{8\|C\|_\infty} + 2\|C\|_\infty\|U_t - p(U_t)\|_F^2 \\ 
& \leq & \frac{\epsilon^2}{200\|C\|_\infty} + 2\|C\|_\infty\|U_t - p(U_t)\|_F^2. \nonumber
\end{eqnarray}
Combining Eq.~\eqref{inequality:obj-progress-third}, Eq.~\eqref{inequality:obj-progress-fourth}, Eq.~\eqref{inequality:obj-progress-fifth} and Eq.~\eqref{inequality:obj-progress-sixth} yields the desired result. 
\end{proof}
Putting Lemma~\ref{lemma:key-descent-supergrad} together with the definition of $p(\bullet)$, we have the following consequence: 
\begin{proposition}\label{prop:obj-progress}
Letting $\{(U_t, \pi_t)\}_{t \geq 1}$ be the iterates generated by Algorithm~\ref{alg:supergrad-simplex}, we have
\begin{eqnarray*}
& & \hspace{- 4 em} \frac{24\|C\|_\infty^2\sum_{t=0}^{T-1}\gamma_{t+1}\|p(U_t) - U_t\|_F^2}{\sum_{t=0}^{T-1} \gamma_{t+1}} \\
& \leq & \frac{\gamma_0^{-1}\Delta_\Phi + 200\gamma_0\|C\|_\infty^3(\gamma_0^2 L_2^2\|C\|_\infty^2 + \gamma_0\|C\|_\infty + \sqrt{k}(\log(T)+1))}{2\sqrt{T}} + \frac{\epsilon^2}{12}, 
\end{eqnarray*}
where $\Delta_\Phi = \max_{U \in \St(d, k)} \Phi(U) - \Phi(U_0)$ is the initial objective gap. 
\end{proposition}
\begin{proof}
By the definition of $p(\bullet)$, we have
\begin{eqnarray*}
& & f(U_t) - f(p(U_t)) + 4\|C\|_\infty\|p(U_t) - U_t\|_F^2 \\
& = & f(U_t) - \left(f(p(U_t)) - 6\|C\|_\infty\|p(U_t) - U_t\|_F^2\right) - 2\|C\|_\infty\|p(U_t) - U_t\|_F^2 \\ 
& \leq & - 2\|C\|_\infty\|p(U_t) - U_t\|_F^2. 
\end{eqnarray*}
Using Lemma~\ref{lemma:key-descent-supergrad}, we have
\begin{eqnarray*}
\Phi(U_{t+1}) & \geq & \Phi(U_t) + 24\gamma_{t+1}\|C\|_\infty^2\|p(U_t) - U_t\|_F^2 - \frac{\gamma_{t+1}\epsilon^2}{12} \\
& & \hspace*{-4em} - 200\gamma_{t+1}^2\|C\|_\infty^3(\gamma_{t+1}^2 L_2^2\|C\|_\infty^2 + \gamma_{t+1}\|C\|_\infty + \sqrt{k}).
\end{eqnarray*}
Rearranging this inequality, we have
\begin{eqnarray*}
24\gamma_{t+1}\|C\|_\infty^2\|p(U_t) - U_t\|_F^2 & \leq & \Phi(U_{t+1}) - \Phi(U_t) + \frac{\gamma_{t+1}\epsilon^2}{12} \\ 
& & \hspace*{-10em} + 200\gamma_{t+1}^2\|C\|_\infty^3(\gamma_{t+1}^2 L_2^2\|C\|_\infty^2 + \gamma_{t+1}\|C\|_\infty + \sqrt{k}). 
\end{eqnarray*}
Summing up over $t = 0, 1, 2, \ldots, T-1$ yields that
\begin{equation*}
\frac{24\|C\|_\infty^2\sum_{t=0}^{T-1}\gamma_{t+1}\|p(U_t) - U_t\|_F^2}{\sum_{t=0}^{T-1} \gamma_{t+1}}  \ \leq \ \frac{\Delta\Phi + 200\|C\|_\infty^3(\sum_{t=1}^T \gamma_t^2(\gamma_t^2 L_2^2\|C\|_\infty^2 + \gamma_t\|C\|_\infty + \sqrt{k}))}{2\sum_{t=1}^T \gamma_t} + \frac{\epsilon^2}{12}. 
\end{equation*}
By the definition of $\{\gamma_t\}_{t \geq 1}$, we have
\begin{equation*}
\sum_{t=1}^T \gamma_t \geq \gamma_0\sqrt{T}, \quad \sum_{t=1}^T \gamma_t^2 \leq \gamma_0^2(\log(T) + 1), \quad \sum_{t=1}^T \gamma_t^3 \leq 3\gamma_0^3, \quad \sum_{t=1}^T \gamma_t^4 \leq 2\gamma_0^4. 
\end{equation*}
Putting these pieces together yields the desired result. 
\end{proof}
We proceed to provide an upper bound for the number of iterations needed to return an $\epsilon$-approximate near-optimal subspace projection $U_t \in \St(d, k)$ satisfying $\Theta(U_t) \leq \epsilon$ in Algorithm~\ref{alg:supergrad-simplex}.
\begin{theorem}\label{Theorem:RSGAN-Total-Iteration}
Letting $\{(U_t, \pi_t)\}_{t \geq 1}$ be the iterates generated by Algorithm~\ref{alg:supergrad-simplex}, the number of iterations required to reach $\Theta(U_t) \leq \epsilon$ satisfies
\begin{equation*}
t \ = \ \bigOtil\left(\frac{k^2\|C\|_\infty^4}{\epsilon^4}\right). 
\end{equation*}
\end{theorem}
\begin{proof}
By the definition of $\Theta(\bullet)$ and $p(\bullet)$, we have $\Theta(U_t) = 12\|C\|_\infty\|p(U_t) - U_t\|_F$. Using Proposition~\ref{prop:obj-progress}, we have
\begin{equation*}
\frac{\sum_{t=0}^{T-1} \gamma_{t+1}(\Theta(U_t))^2}{\sum_{t=0}^{T-1} \gamma_{t+1}} \ \leq \ \frac{3\gamma_0^{-1}\Delta\Phi + 600\gamma_0\|C\|_\infty^3(\gamma_0^2L_2^2\|C\|_\infty^2 + \gamma_0\|C\|_\infty + \sqrt{k}(\log(T)+1))}{\sqrt{T}} + \frac{\epsilon^2}{2}.  
\end{equation*}
Furthermore, by the definition $\Phi(\bullet)$, we have
\begin{eqnarray*}
|\Phi(U)| & \leq & \max\limits_{U' \in \St(d, k)} \ |f(U') + 6\|C\|_\infty\|U' - U\|_F^2| \\
& \leq & \max\limits_{U \in \St(d, k)} \max\limits_{U' \in \St(d, k)} \ |f(U') + 6\|C\|_\infty\|U' - U\|_F^2| \\
& \leq & \max\limits_{U \in \St(d, k)} |f(U)| + 12k\|C\|_\infty. 
\end{eqnarray*}
By the definition of $f(\bullet)$, we have $\max_{U \in \St(d, k)} |f(U)| \leq \|C\|_\infty$. Putting these pieces together with $k \geq 1$ implies that $|\Phi(U)| \leq 20k\|C\|_\infty$. By the definition of $\Delta_\Phi$, we conclude that $\Delta_\Phi \leq 40k\|C\|_\infty$. Given that $\gamma_0 = 1/\|C\|_\infty$ and $\Theta(U_t) > \epsilon$ for all $t = 0, 1, \ldots, T-1$, the upper bound $T$ must satisfy
\begin{equation*}
\epsilon^2 \ \leq \ \frac{240k\|C\|_\infty^2 + 1200\|C\|_\infty^2(L_2^2 + \sqrt{k}\log(T) + \sqrt{k} + 1)}{\sqrt{T}}. 
\end{equation*}
This implies the desired result. 
\end{proof}
Equipped with Theorem~\ref{Theorem:RSGAN-Total-Iteration} and Algorithm~\ref{alg:supergrad-simplex}, we establish the complexity bound of Algorithm~\ref{alg:supergrad-simplex}.
\begin{theorem}\label{Theorem:RSGAN-Total-Complexity}
The RSGAN algorithm (cf. Algorithm~\ref{alg:supergrad-simplex}) returns an $\epsilon$-approximate pair of near-optimal subspace projection and optimal transportation plan of computing the PRW distance in Eq.~\eqref{prob:main} (cf. Definition~\ref{def:near-optimal-pair}) in
\begin{equation*}
\bigOtil\left(\frac{n^2(n + d)\|C\|_\infty^4}{\epsilon^4}\right)
\end{equation*}
arithmetic operations. 
\end{theorem}
\begin{proof}
First, Theorem~\ref{Theorem:RSGAN-Total-Iteration} implies that the iteration complexity of Algorithm~\ref{alg:supergrad-simplex} is
\begin{equation}\label{RSGAN-iteration}
\bigOtil\left(\frac{k^2\|C\|_\infty^4}{\epsilon^4}\right). 
\end{equation}
This implies that $U_t$ is an $\epsilon$-approximate near-optimal subspace projection of problem~\eqref{prob:Stiefel-nonsmooth}. Furthermore, $\widehat{\epsilon} = \min\{\epsilon, \epsilon^2/144\|C\|_\infty\}$. Since $\pi_{t+1} \leftarrow \textsc{OT}(\{(x_i, r_i)\}_{i \in [n]}, \{(y_j, c_j)\}_{j \in [n]}, U_t, \widehat{\epsilon})$, we have $\pi_{t+1} \in \Pi(\mu, \nu)$ and $\langle U_tU_t^\top, V_{\pi_{t+1}} - V_{\pi_t^\star}\rangle \leq \widehat{\epsilon} \leq \epsilon$. This implies that $\pi_{t+1}$ is an $\epsilon$-approximate optimal transportation plan for the subspace projection $U_t$. Therefore, we conclude that $(U_t, \pi_{t+1}) \in \St(d, k) \times \Pi(\mu, \nu)$ is an \emph{$\epsilon$-approximate pair of near-optimal subspace projection and optimal transportation plan} of problem~\eqref{prob:main}. 

The remaining step is to analyze the complexity bound. Note that the most of software packages, e.g., \textsc{POT}~\citep{Flamary-2017-Pot}, implement the OT subroutine using a variant of the network simplex method with a block search pivoting strategy~\citep{Damian-1991-Minimum, Bonneel-2011-Displacement}. The best known complexity bound is provided in~\citet{Tarjan-1997-Dynamic} and is $\bigOtil(n^3)$. Using the same argument in Theorem~\ref{Theorem:RGAS-RAGAS-Total-Complexity}, the number of arithmetic operations at each loop is
\begin{equation}\label{RSGAN-arithmetic-operation}
\bigOtil\left(n^2 dk + dk^2 + k^3 + n^3\right). 
\end{equation}
Putting Eq.~\eqref{RSGAN-iteration} and Eq.~\eqref{RSGAN-arithmetic-operation} together with $k = \bigOtil(1)$ yields the desired result.
\end{proof}
\begin{remark}
The complexity bound of Algorithm~\ref{alg:supergrad-simplex} is better than that of Algorithm~\ref{alg:grad-sinkhorn} and~\ref{alg:adagrad-sinkhorn} in terms of $\epsilon$ and $\|C\|_\infty$. This makes sense since Algorithm~\ref{alg:supergrad-simplex} only returns an $\epsilon$-approximate pair of \textit{near-optimal} subspace projection and optimal transportation plan which is weaker than an $\epsilon$-approximate pair of optimal subspace projection and optimal transportation plan. Furthermore, Algorithm~\ref{alg:supergrad-simplex} implements the network simplex method as the inner loop which might suffer when $n$ is large and yield unstable performance in practice.  
\end{remark}

\end{document}